\documentclass[final,3p,times,sort,compress]{elsarticle}

\usepackage{amssymb}
\usepackage{amsthm}
\usepackage{bm}
\usepackage{tikz}
\usepackage{caption}
\usetikzlibrary{positioning, calc}
\usepackage{booktabs}
 \usepackage{longtable}
\usepackage{amsmath}
\usepackage{array}
\usepackage{float}

\usepackage{amsmath}
\usepackage{bbm}
\usepackage{svg}

\usepackage{xurl}
\usepackage[breaklinks=true]{hyperref}

\usepackage{booktabs}

\usepackage{soul}

\usepackage{lipsum} 
\usepackage[linesnumbered,ruled]{algorithm2e}
\SetAlCapSty{}

\usepackage[export]{adjustbox}
\usepackage{xcolor}
\usepackage{comment}
\usepackage{graphicx}
\usepackage{subcaption}
\usepackage{array}

\definecolor{dark_blue}{HTML}{1f77b4}
\definecolor{light_blue}{HTML}{aec7e8}
\definecolor{dark_orange}{HTML}{ff7f0e}
\definecolor{light_orange}{HTML}{ffbb78}
\definecolor{dark_green}{HTML}{2ca02c}
\definecolor{light_green}{HTML}{98df8a}
\definecolor{dark_red}{HTML}{d62728}
\definecolor{light_red}{HTML}{ff9896}
\definecolor{dark_purple}{HTML}{9467bd}
\definecolor{light_purple}{HTML}{c5b0d5}
\definecolor{dark_brown}{HTML}{8c564b}
\definecolor{light_brown}{HTML}{c49c94}
\definecolor{dark_pink}{HTML}{e377c2}
\definecolor{light_pink}{HTML}{f7b6d2}
\definecolor{dark_gray}{HTML}{7f7f7f}
\definecolor{light_gray}{HTML}{c7c7c7}
\definecolor{dark_olive}{HTML}{bcbd22}
\definecolor{light_olive}{HTML}{dbdb8d}
\definecolor{dark_cyan}{HTML}{17becf}
\definecolor{light_cyan}{HTML}{9edae5}
\definecolor{white}{HTML}{FFFFFF}
\definecolor{black}{HTML}{000000}

 \tikzset{
     mynode/.style={circle, draw=black, fill=light_gray, minimum size=5mm, inner sep=0pt, font=\scriptsize},
     rednode/.style={mynode, fill=light_red},
     bluenode/.style={mynode, fill=light_blue},
     greennode/.style={mynode, fill=light_green},
     orangenode/.style={mynode, fill=light_orange},
 }

\newcommand{\PreserveBackslash}[1]{\let\temp=\\#1\let\\=\temp}
\newcolumntype{C}[1]{>{\PreserveBackslash\centering}p{#1}}
\newcolumntype{R}[1]{>{\PreserveBackslash\raggedleft}p{#1}}
\newcolumntype{L}[1]{>{\PreserveBackslash\raggedright}p{#1}}

\newtheorem{theorem}{Theorem}
\newtheorem{definition}{Definition}
\newtheorem{proposition}{Proposition}

\newtheorem{assumption}{Assumption}

\newcommand{\review}[1]{{\color{red} R: #1}}
\newcommand{\answer}[1]{{\color{orange} ANS: #1}}

\newcommand{\llf}[1]{{\color{teal} LLF: #1}}

\newcommand{\resolution}{\gamma} 
\newcommand{\difficulty}{\lambda}
\newcommand{\noise}{\eta}

\newcommand{\numer}{b}
\newcommand{\denom}{c}
\usepackage{framed}

\newcommand{\graph}{G}
\newcommand{\vertices}{V}
\newcommand{\edges}{E}
\newcommand{\nvertices}{n}
\newcommand{\medges}{m}
\newcommand{\anode}{i} 
\newcommand{\bnode}{j} 

\newcommand{\partition}{\pi}
\newcommand{\community}{\mathcal{C}}
\newcommand{\communityAlt}{\mathcal{S}}
\newcommand{\nodecomm}{\sigma}

\newcommand{\potentialNode}{\varphi}
\newcommand{\potentialCoalition}{\phi}
\newcommand{\potentialPartition}{\Phi}

\newcommand{\clusterA}{\mathcal{A}}
\newcommand{\clusterB}{\mathcal{B}}

\newcommand{\degriA}{d_{\anode}^{\clusterA}}
\newcommand{\degriB}{d_{\anode}^{\clusterB}}

\newcommand{\cdegriA}{\hat{d}_{\anode}^{\clusterA}}
\newcommand{\cdegriB}{\hat{d}_{\anode}^{\clusterB}}

\newcommand{\degrik}{d_{\anode}^k}
\newcommand{\cdegrik}{\hat{d}_{\anode}^k}

\newcommand{\accuracy}{\mathbb{A}}
\newcommand{\robustness}{R}

\journal{Physica A: Statistical Mechanics and its Applications}

\begin{document}

\begin{frontmatter}

\author[label1]{Lucas Lopes Felipe}
\author[label2]{Konstantin Avrachenkov}
\author[label1]{Daniel Sadoc Menasché}
\affiliation[label1]{organization={Federal University of Rio de Janeiro (UFRJ)},
            city={Rio de Janeiro},
            country={Brazil}}
\affiliation[label2]{organization={National Institute for Research in Digital Science and Technology (Inria)},
            city={Sophia Antipolis},
            country={France}}

\title{From Leiden to Pleasure Island: The Constant Potts Model for\\Community Detection as a Hedonic Game}

\begin{abstract}
Community detection is   one of the fundamental problems in data science which consists of partitioning nodes into disjoint communities.  We present a game-theoretic perspective on the Constant Potts Model (CPM) for partitioning networks into disjoint communities, emphasizing its efficiency, robustness, and accuracy.
Efficiency: We reinterpret CPM as a potential hedonic game by decomposing its global Hamiltonian into local utility functions, where  the  local utility gain of each agent  matches the corresponding increase in   global utility. Leveraging this equivalence, we prove that local optimization of the CPM objective via better-response dynamics converges in pseudo-polynomial time to an equilibrium partition.  
Robustness: We introduce and relate two stability criteria: a strict criterion based on a novel notion of robustness, requiring nodes to simultaneously maximize neighbors and minimize non-neighbors within communities, and a relaxed utility function based on a weighted sum of these objectives, controlled by a resolution parameter. 
Accuracy: In community tracking scenarios, where initial partitions are used to bootstrap the Leiden algorithm with partial ground-truth information, our experiments reveal that robust partitions yield higher accuracy in recovering ground-truth communities.
\end{abstract}

\begin{keyword}
Network Partitioning \sep
Community Detection \sep
Constant Potts Model \sep
Leiden Algorithm \sep
Computational Social Choice \sep
Hedonic Games
\end{keyword}

\end{frontmatter}

\section{Introduction}

Clustering is a fundamental unsupervised machine learning technique for grouping similar data points without relying on pre-existing labels. Its utility spans a wide array of applications~\cite{shen2013community}. 
However, the power of clustering is matched by its inherent complexity~\cite{boccaletti2006complex}. The notion of ``similarity'' is often context-dependent, and different algorithms can produce multiple, distinct, yet equally valid groupings. 

Community detection extends the clustering paradigm to network data, where the goal is to partition the nodes of a graph into disjoint communities based on their connectivity patterns. This approach is useful for any problem that can be modeled as a graph where the objective is to split nodes into non-overlapping groups. The spectrum of possible partitions is vast, ranging from the \textit{grand coalition}, where all nodes form a single community, to the \textit{singleton partition}, where each vertex constitutes its own community. 

To visualize this enormous search space, consider the graph in Figure~\ref{fig:metagraph}(a)  and its corresponding \textit{metagraph} in Figure~\ref{fig:metagraph}(b). In the metagraph, each metavertex represents a unique partition of the original network, and an edge connects two partitions
if one can be reached from the other by a unilateral move wherein a single vertex of the original graph changes its community.%
\footnote{
We consider only sequential, unilateral moves, excluding simultaneous-move scenarios to simplify the model and focus on convergence.
}
Each such edge represents a decision for that moving vertex, whose choice is governed by a trade-off between two competing objectives:
\begin{itemize}
  \item {Maximize ``\emph{friends}'' within the community}: 
  choose the community with the highest number of neighbors, i.e., direct edges to the agent.
  \item {Minimize ``\emph{strangers}'' within the community}: 
  choose the community with the fewest number of nodes to which the agent is not directly connected.
\end{itemize}
These goals align with the widely accepted notion of internal cohesion and external sparsity in community detection. Increasing the number of intra-community edges and decreasing the presence of non-neighbors strengthens the community’s internal density.

\begin{figure}[tbp]
    \centering
    \begin{tabular}{@{}p{0.2\linewidth}@{}p{0.8\linewidth}@{}}
         \raisebox{\height}{
        \begin{minipage}[t]{\linewidth}
            \centering
            \begin{tikzpicture}[every node/.style={circle,draw,fill=light_cyan,minimum size=3mm}, node distance=1cm]
                \node (0) at (90:1cm) {0};
                \node (1) at (0:1cm) {1};
                \node (2) at (270:1cm) {2};
                \node (3) at (180:1cm) {3};

                \draw (0) -- (1);
                \draw (0) -- (2);
                \draw (0) -- (3);
                \draw (1) -- (2);
            \end{tikzpicture}
        \end{minipage} }
        &
        \begin{minipage}[t]{\linewidth}
            \centering
            \includegraphics[width=\linewidth]{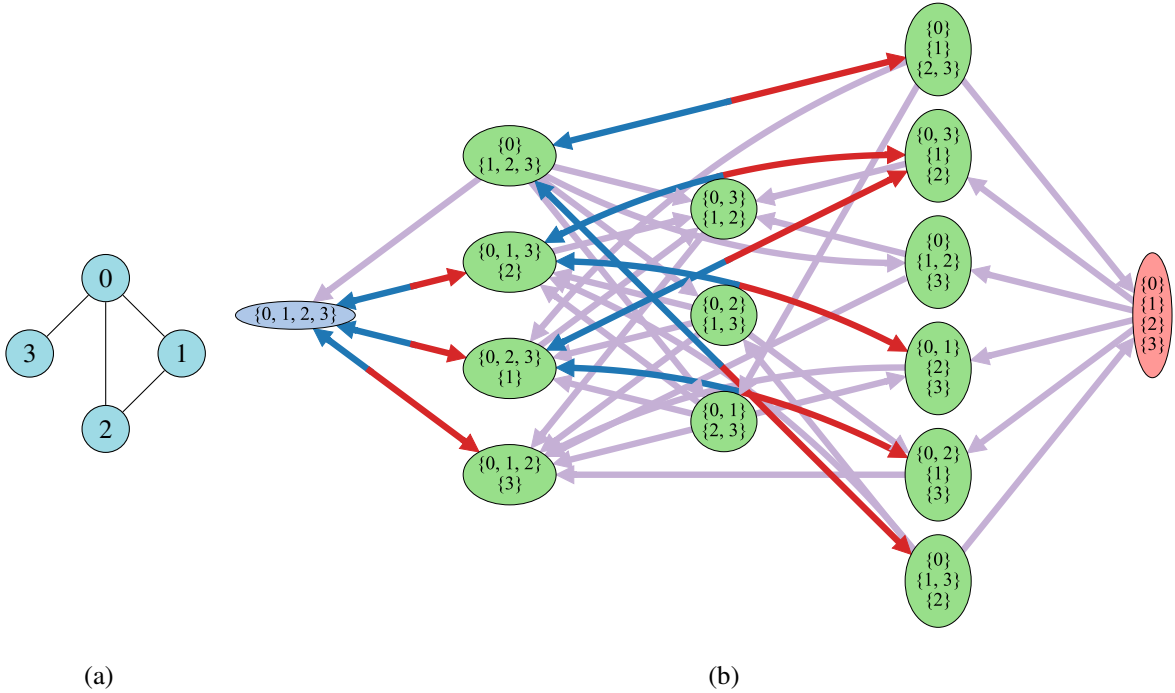}
        \end{minipage}
        \\
        \centering (a) & \centering (b) \\
    \end{tabular}
    \caption{(a) Illustrative graph and (b) corresponding  metagraph of partitions  (see~\ref{sec:metagraph} for details). The metagraph represents each unique partition of the original network as a vertex. For instance, the blue node represents the grand coalition, while the red node represents the singleton partition. Edges connect partitions that are reachable from each other by a single vertex move.
    \textbf{Unidirectional (purple) edges} denote clear preferences consistent with the decision tree in Figure~\ref{fig:tree}.
    \textbf{Bidirectional (red/blue) edges} highlight frustrated choices, where the \textit{Familiarity Index} (Eq.~\eqref{eq:familiarity_index}) quantifies the trade-off between gaining friends (blue path) and avoiding strangers (red path).}
    \label{fig:metagraph}
\end{figure}

Figure~\ref{fig:tree} illustrates a decision tree for a vertex making this comparison. Cases leading to purple boxes indicate a clear preference, translating into a unidirectional edge in the metagraph (Figure~\ref{fig:metagraph}(b)).
Inspired by the concept of ``frustration'' in the Ising model~\cite{brush1967ising}, we describe a choice as ``frustrated'' when the preferences of an agent for different communities are in conflict. This occurs, for instance, when one candidate community offers more friends while another offers fewer strangers. In such scenarios, a node cannot simultaneously optimize both objectives, and the move represents a \textbf{frustrated choice}, depicted by bidirectional edges in the metagraph.

Most community detection algorithms agree on outcomes where a clear preference exists (the purple boxes in Figure~\ref{fig:tree}). They diverge, however, in ``frustrated'' cases. Some algorithms may prioritize a community with more friends, while others might favor one with fewer strangers. This distinction gives rise to two perspectives on the problem: the \textit{robustness} perspective and the \textit{resolution} perspective.

\begin{figure}
    \centering
    \includegraphics[width=\textwidth]{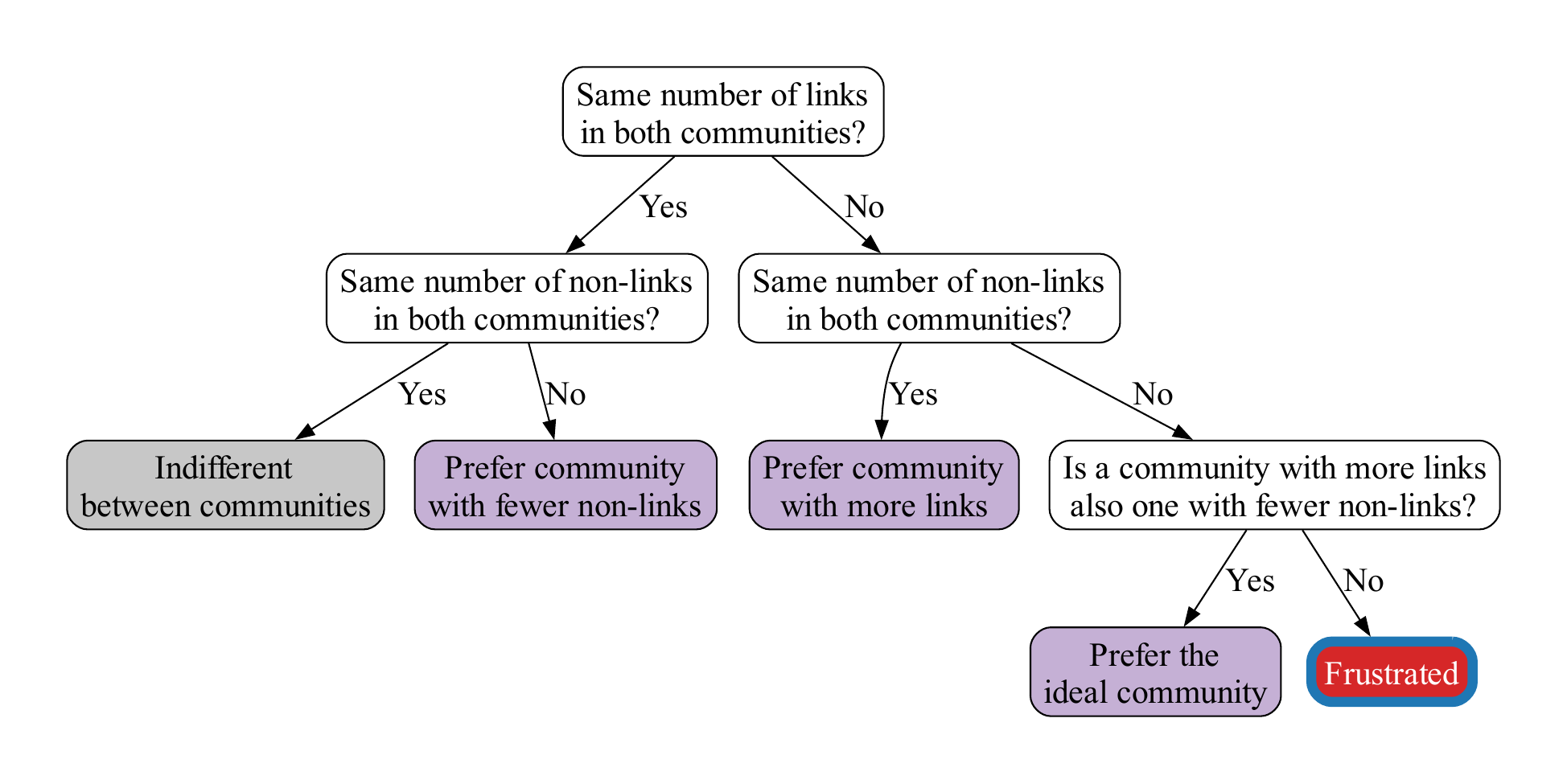}
    \captionof{figure}{Decision tree from a node's perspective when comparing two communities.
    Purple boxes indicate unambiguous choices, where one partition is clearly preferred over another because the move improves at least one metric (friends/strangers) without worsening the other;
    The red box with blue border highlights frustrated choices, where a move to a different community involves a trade-off: one option offers more friends but also more strangers, while the other offers the reverse.}
    \label{fig:tree}
\end{figure}

\paragraph{\textbf{Robustness perspective}} 
It focuses on a metric that quantifies scenarios where a clear preference exists. In this view, a node's stability determines its classification. A node is considered \textbf{robust} if its current community is unambiguously preferred over all other communities reachable through a unilateral move. This means its community simultaneously offers at least as many ``friends'' (neighbors) and no more ``strangers'' (non-neighbors) compared to any other available option. Conversely, any node that does not meet this strict criterion is considered non-robust, indicating it may have an incentive to move.

A node's classification is therefore based on pairwise comparisons between its current community and every other available option, following the logic depicted in Figure 2. Another way to conceptualize this choice is to imagine that each node holds two ranked lists of communities: one ordered by the number of friends gained and the other by the number of strangers avoided. A node is robust if the top-ranked community is the same on both lists. While this analogy is intuitive, a more rigorous definition is needed to handle potential ties for the top rank.

Formally, for a graph $\graph = (\vertices, \edges)$ with $\nvertices = |\vertices|$ nodes and $\medges = |\edges|$ edges, a community (or coalition) is a set $\community$ of nodes, $\community \subseteq \vertices$ (see  notation in~\ref{sec:notation}, Table~\ref{tab:notation}). A community detection algorithm partitions the network into subsets $\partition = \{ \community_1,\community_{2},\dots,\community_K \}$ such that $\community_{k} \cap \community_{\ell} = \emptyset$, for all $k,\ell \in \{1, \ldots, K\}$ where ${k} \neq {\ell}$, and $\cup_{k=1}^K \community_k = \vertices$. 

In this work, unless otherwise noted, we work under the following assumption on the number of communities:
\begin{assumption}[Number of  Communities]
The number of communities $K$ is fixed and given. 
\label{ass:mainassumption}
\end{assumption}
 Noting that some communities may be empty, the maximum number of empty communities is $K-1$. 
 For a given partition $\partition$ of $K$ communities, the \textbf{Robustness} of $\partition$ is the fraction of robust nodes:
\begin{equation}
    \robustness(\graph, \partition) = \frac{1}{\nvertices} \sum_{\anode \in \vertices} R(\anode, \partition) \quad \text{where} \quad R(i,\partition)= \begin{cases} 
    1 & \text{if } d^{\nodecomm_{\anode}}_{\anode} \geq d^{k}_{\anode} \text{ and } \hat{d}^{\nodecomm_{\anode}}_{\anode} \leq \hat{d}^{k}_{\anode} \ \forall k \\
    0 & \text{otherwise}
    \end{cases}.
    \label{eq:robustness}
\end{equation}
Here, $d^{k}_\anode$ and $\hat{d}^{k}_\anode$ denote the number of neighbors and non-neighbors of node $i$ in community $\community_k$, and $\sigma_i$ is the index of node $i$ current community. 

\paragraph{\textbf{Familiarity Index}} We derive a social metric to quantify the possible trade-off when a node $\anode$ compares community $\community_A$ to community $\community_B$.
The intuition is to measure what proportion of the total change in community composition is attributable to the change in friends.
\begin{enumerate}
    \item \textbf{Friend-Gain:} The net change in friends is $\Delta d = d^{\community_B}_{\anode} - d^{\community_A}_{\anode}$;
    \item \textbf{Total Change:} The total change in perceived community size is the sum of the change in friends and strangers: $\Delta d + \Delta \hat{d}$, where $\Delta \hat{d} = \hat{d}^{\community_B}_{\anode} - \hat{d}^{\community_A}_{\anode}$.
\end{enumerate}
The Familiarity Index, $F_\anode(\community_A, \community_B)$, is the ratio of the friend-gain to this total change:
\begin{equation}
    F_\anode(\community_A, \community_B) = \frac{\Delta d}{\Delta d + \Delta \hat{d}}.
    \label{eq:familiarity_index}
\end{equation}
This index reveals the nature of a node's decision.  
Given a node $i$ in $\community_A$, we  consider two cases:
\begin{itemize}
    \item If $F_\anode(\community_A, \community_B) \geq 1$ or $F_\anode(\community_A, \community_B) \leq 0$, the choice is clear, corresponding to unidirectional edges in Figure~\ref{fig:metagraph}(b) and  purple boxes in Figure~\ref{fig:tree}.
    \item If $0<F_{i} (\community_{A},\community_{B})<1$, the decision to move to $\community_{B}$ is a  frustrated choice: the node must balance an increase in friends $(\Delta d>0)$ with an increase in strangers $(\Delta\hat{d}>0)$, corresponding to bidirectional edges in Figure~\ref{fig:metagraph}(b), where the Familiarity Index acts as the decision's splitting point (see  red  box in Figure~\ref{fig:tree}). 
\end{itemize}  
In what follows, we relate the proposed Familiarity Index to the well-known concept of resolution, highlighting its role in addressing the fundamental trade-off between friends and strangers.

\paragraph{\textbf{Resolution perspective}}
It aims to resolve these trade-offs by combining the two \emph{max-min} objectives into a single quality function, where a move is beneficial if the quality gain is positive.
Following the literature on the Constant Potts Model (CPM)~\cite{traag2011cpm}, the quality, also referred to in the literature as weight, or value, of a node pair $(\anode,\bnode)$ is defined by its resolution parameter, $\resolution$, as
\begin{equation}
    v_{\anode\bnode} = (A_{\anode\bnode} - \resolution)\,\delta(\nodecomm_{\anode}= \nodecomm_{\bnode}),
    \label{eq:pair_value}
\end{equation}
where $A_{\anode\bnode}$ is the entry in line $\anode$ and column $\bnode$ of the adjacency matrix, the Kronecker delta $\delta(\nodecomm_{\anode}= \nodecomm_{\bnode})=1$ if $\anode$ and $\bnode$ are assigned to the same community (and $0$ otherwise), and $0 \leq \resolution \leq 1$.
To simplify notation, we set $A_{\anode\anode}=\resolution$. 
Then, 
\begin{equation}
    v_{\anode\bnode}= 
    \begin{cases}
    1 - \resolution, & \text{if } A_{\anode\bnode}=1 \textrm{ and } \sigma_{\anode} = \sigma_{\bnode} \;(\text{neighbors or ``\emph{friends}'' within community}),\\
    - \resolution, & \text{if } A_{\anode\bnode}=0 \textrm{ and } \sigma_{\anode} = \sigma_{\bnode}  \;(\text{non-neighbors  or ``\emph{strangers}'' within community}),\\
    0, & \text{otherwise.}
    \end{cases}
    \label{eq:pair_value_detailed}
\end{equation}
When $\resolution = 0$, one ignores non-neighbor penalties, favoring one large ``grand coalition.'' In contrast, when $\resolution = 1$, non-neighbor penalties dominate, favoring singleton partitions.
This way, the resolution parameter controls cluster granularity, e.g., enabling the detection of small communities within large networks.


Robust nodes have no incentive to deviate for any $\resolution \in [0,1]$. However, for a non-robust node, a move to a different community may represent a  frustrated choice. The decision for such a move is governed by comparing the Familiarity Index to $\resolution$,
where a move is preferred if:
\begin{itemize}
    \item $\resolution < F_\anode$, the node prioritizes the community with \textit{more friends}, following the blue arrows in Figure~\ref{fig:metagraph};
    \item $\resolution > F_\anode$, the node prioritizes the community with \textit{fewer strangers}, following  the red arrows in Figure~\ref{fig:metagraph}.
\end{itemize}
The parameter $\resolution$ thus directly controls how to deal with  frustrated  choices.

The quality (or potential) of a node $i$ in community $\community_k$ is the sum of the qualities of the node pairs $v_{ij}$ for all $j \in \community_k$, and  the quality of a community and of  a partition  can be similarly defined.  
The introduction of a consistent quality to each node, such as the one derived from CPM, imposes a clear direction on all possible   moves in the metagraph, turning it into a directed acyclic graph (DAG) whose sinks represent local optima on the quality function.

This formulation naturally suggests a simple \textit{better-response dynamic}: starting from an arbitrary partition, an algorithm iteratively allows a self-interested node to make a selfish move to a neighboring partition if doing so improves its individual utility, continuing until a sink in the metagraph is reached where no node has an incentive to move.%
\footnote{Cycles of partitions are impossible in the setup considered in this paper, which corresponds to an additive potential game. In such game,   each selfish move strictly increases the global  potential, making a return to a previous partition logically contradictory,  unlike in other social choice models where the Condorcet paradox can occur~\cite{gehrlein2006condorcet}.} 



Among our contributions, we will show that under the formulation introduced in Sections~\ref{sec:literature_review} and~\ref{sec:analytical_contributions} the considered better-response dynamic always converges in pseudo-polynomial time to a local optimum that is  a candidate solution to the community detection problem.  

\paragraph{\textbf{Bridging robustness and resolution}}
The robustness and resolution perspectives instantiate  two ways of handling the \emph{max-min} nature of the community detection problem:
\begin{itemize}
    \item A \emph{robustness-centric} approach requires that only a fraction of the nodes satisfy the strict dual criterion (Eq.~\eqref{eq:robustness}), leaving other nodes unconstrained;
    \item A \emph{resolution-centric} approach requires \emph{all} nodes to be at a local optimum for some $\resolution \in [0,1]$ (Eq.~\eqref{eq:pair_value}), but relaxes the emphasis on strict max-min satisfaction for each individual node.
\end{itemize}
In this work, we analytically relate these approaches and address the convergence concerns for the CPM~\cite{traag2011cpm}, a well-established quality function from statistical physics that serves as the default objective for state-of-the-art optimizers like the Leiden algorithm~\cite{traag2019leiden}.
We employ a \emph{resolution-centric approach} to find a local optimal partition. To do so, we set the resolution parameter $\resolution$ heuristically to the \emph{edge density of the graph}, following a recommendation from~\cite{avrachenkov2017cooperative}. Once a local optimum partition is identified, we then evaluate its \emph{robustness}. This indicates the fraction of optimal nodes that would remain stable, even if we were to vary the resolution from the value at which the partition was found.

\paragraph{\textbf{Contributions}}  We summarize the key contributions of this paper as follows:
\begin{itemize}
    \item \textbf{Game-theoretic reinterpretation of CPM.} 
    We reinterpret the  CPM as a potential hedonic game, where each vertex is modeled as a self-interested agent maximizing its utility~\cite{avrachenkov2017cooperative, brandt2016handbook, dreze1980hedonic}.%
    \footnote{The term ``hedonic'' originates from the Ancient Greek \textit{hēdonikos} (pleasurable) and \textit{hēdonē} (pleasure). In hedonic games, agents greedily maximize satisfaction based on a utility function. The title alludes to both Pinocchio’s Pleasure Island and Leiden, the city where the algorithm was developed~\cite{traag2011cpm, traag2019leiden}.} 
    This framing allows the CPM quality function to be viewed as a sum of individual utilities, providing a game-theoretic lens for community detection that computationally   scales under  distributed computation.

    \item \textbf{Convergence guarantees for better-response dynamics.} 
    We prove that better-response dynamics over the CPM objective (Eq.~\eqref{eq:cpm}) converge in pseudo-polynomial time to an equilibrium (Definition~\ref{def:solution}), where no agent has an incentive to deviate (Theorem~\ref{THEO:EQUILIBRIUM}). The convergence is guaranteed by bounding the quality improvement at each step and leveraging the structure of the CPM Hamiltonian.

    \item \textbf{Theoretical characterization of robustness.} 
    We introduce a robustness measure (Eq.~\eqref{eq:robustness}) and establish formal connections between equilibrium partitions and robustness properties. These results (Theorems~\ref{THEO:RANGE_EQ} and~\ref{THEO:ROBUST_EQ}) link robustness (Eq.~\eqref{eq:robustness}) to the resolution parameter $\resolution$. 

    \item \textbf{Empirical evaluation of efficiency, robustness, and accuracy.} 
    We conduct numerical experiments on synthetic networks with ground-truth communities to evaluate three fundamental aspects of community detection: \textit{efficiency} (answering ``how long does it take, in the worst case, to reach an equilibrium partition?"); \textit{robustness} (evaluating ``how resilient is a partition to parameter changes?" using Eq.~\eqref{eq:robustness}); and \textit{accuracy} (determining ``how close is a partition to the ground truth?" via the \emph{Adjusted Rand index (ARI)}~\cite{hubert1985comparing}). We analyze these for both the ground truth and partitions found by algorithms such as the Leiden algorithm~\cite{traag2019leiden}.
\end{itemize}

\paragraph{\textbf{Outline}} 
The remainder of this paper is organized as follows.  {Section~\ref{sec:literature_review}} introduces the CPM and its reinterpretation as a hedonic game, situating our work within the broader literature on coalitional game theory and community detection. Then,  Section~\ref{sec:analytical_contributions} presents our analytical results, including convergence proofs for better-response dynamics and the theoretical connection between robustness and resolution.
 Section~\ref{sec:empirical_analysis}  reports our experimental methodology and results, addressing the efficiency, robustness, and accuracy of detected partitions.
Finally,  Section~\ref{sec:conclusion} concludes  with  broader implications and directions for future work.

\section{Background and Literature Review}
\label{sec:literature_review}

Complex systems, ranging from biological ecosystems to social and technological networks, are composed of interacting components whose collective behavior often exhibits emergent properties not evident at the individual level. Understanding how such interactions give rise to functional organization is a central challenge, and network science provides a powerful analytical framework for this purpose~\cite{dorogovtsev2003evolution}. Within this framework, the detection and analysis of \textit{community structure} plays a pivotal role, as it reveals modular organization, supports efficient information flow, and helps mitigate systemic risks such as cascading failures~\cite{shen2013community}.

The foundational work of Girvan and Newman~\cite{girvan2002community} marked a turning point in the algorithmic study of community detection, inspiring a surge of methods that bridge theoretical insights with real-world applications. Despite this progress, community detection remains a computationally challenging problem. Impossibility results in synthetic benchmarks~\cite{papadimitriou2007complexity, massoulie2014community, young2018sbm, xu2020optimal} have demonstrated inherent limits to accurate recovery, especially in sparse or noisy networks.

To navigate these challenges, diverse algorithmic strategies have been proposed. Some prioritize efficiency and scalability, such as the widely-used Louvain~\cite{blondel2024louvain} and Leiden~\cite{traag2019leiden} algorithms, which optimize modularity-based objectives. Others take a more theoretical stance, employing asymptotic analysis under generative models like the Stochastic Block Model (SBM)~\cite{abbe2015exact, xu2020optimal, lancichinetti2012consensus}. In this context, polynomial-time methods developed by Hajek and colleagues~\cite{hajek2019community, hajek2018recovering} provide valuable insights into the limits of tractability.

In parallel, bioinspired heuristics have emerged as a compelling alternative, offering adaptive and distributed strategies for uncovering community structures~\cite{saoud2019networks, saoud2023nature, saoud2024community}. Another fruitful direction is the application of maximum likelihood principles~\cite{abbe2015exact, mazalov2018comparing, prokhorenkova2019community}, which frame community detection as a probabilistic inference problem.

More recently, game-theoretic approaches have gained momentum, particularly those grounded in cooperative game theory. Hedonic games~\cite{avrachenkov2017cooperative} model community formation as a process driven by individual preferences over groupings, offering a principled lens through which to analyze equilibrium formation, stability, and agent incentives.

To establish the theoretical foundations for our contributions, this section is organized into two conceptual stages. Section~\ref{sec:from_ising_to_cpm} revisits the development of quality functions for community detection, tracing their origins in statistical physics, from the Ising model~\cite{brush1967ising} to the CPM. The CPM stands out for addressing the resolution-limit problem and enabling locally consistent communities, making it a cornerstone of our approach.

Then, in Section~\ref{sec:from_cpm_to_hedonic}, we reformulate the CPM objective using the language of Hedonic Game Theory. By decomposing the global Hamiltonian into agent-level utility functions, we demonstrate that CPM-induced community formation constitutes a potential game, in which each agent’s local decision aligns with global optimization. This reinterpretation forms the basis of our convergence guarantees, discussed in the next section.

\subsection{From the Ising Model to the Constant Potts Model}
\label{sec:from_ising_to_cpm}

This subsection outlines the evolution of community detection algorithms, starting from the foundational Lenz–Ising model, analogous to two-community detection (Section~\ref{sec:ising}). We then explore the emergence and limitations of geometric frustration (Section~\ref{sec:frustration}), leading to the introduction of the more generalized Potts model for $K$ communities (Section~\ref{sec:potts}). We then discuss how Modularity (Section~\ref{sec:modularity}) derives from the Potts model, examine its resolution-limit problem (Section~\ref{sec:resolution_limit}), and finally present the  CPM  as a resolution-limit-free approach that forms the theoretical basis of our work (Section~\ref{sec:cpm}).

\subsubsection{The Lenz–Ising Model}
\label{sec:ising}

The challenge of partitioning networks into meaningful communities has deep roots in statistical physics, with early concepts providing the foundation for many modern algorithms. The Lenz–Ising model, introduced in 1925 by physicists Ernst Ising and Wilhelm Lenz, represents a seminal step in this direction. Originally formulated to explain ferromagnetism, it describes a system of interacting ``spins'' on a lattice or graph, where each spin (vertex) can assume one of two states (communities), e.g., $+1$ or $-1$.

The system's total energy is quantified by the Lenz–Ising Hamiltonian,
\begin{equation}
\label{eq:ising}
\mathcal{H}_{\text{Ising}}(\sigma) = - \sum_{i \neq j} J_{ij} \sigma_i \sigma_j.
\end{equation}
Here, $\sigma = (\sigma_1, \dots, \sigma_n)$ represents the partition, where each component $\sigma_i \in \{-1, +1\}$ represents the state of node $i$. $J_{ij}$ represents the interaction strength between nodes $i$ and $j$, with the convention that $J_{ij}=0$ whenever $i$ and $j$ are not neighbors.
The objective in the Ising model is to find a configuration of spins that minimizes the total energy of the system. This state of lowest energy is known as the  ``ground state'', and other configurations with higher energy are referred to as ``excited states''.
In the ferromagnetic case, where $J_{ij}>0$, the energy is minimized when neighboring spins align (both $+1$ or both $-1$). Conversely, in the antiferromagnetic case, where $J_{ij}<0$, the energy is minimized when neighboring spins are anti-aligned (one $+1$ and the other $-1$).

Note that if we compute $\mathcal{H}_{\text{Ising}}$ across all possible partitions represented as nodes in the metagraph (Figure~\ref{fig:metagraph}), the resulting directed metagraph (which points toward partitions of higher quality) will have the grand coalition as the ground state in the ferromagnetic case, since all edges connect nodes within the same community. In contrast, in the antiferromagnetic case, the singleton partition becomes the ground state, as all edges link nodes from different communities. However, the singleton partition in Figure~\ref{fig:metagraph} comprises four communities, one for each node in the original graph, whereas the Ising model only accommodates two spin states. Therefore, in the antiferromagnetic case, frustration inevitably arises.

\subsubsection{Geometric Frustration in Networks}
\label{sec:frustration}

Frustration arises in the Ising model when the network's topology prevents all local interaction preferences from being simultaneously satisfied, leading to multiple configurations sharing the same minimum energy. Consider, for instance, an antiferromagnetic Ising model on a triangle. 
In the triangle, at least one interaction will always be \textit{frustrated} or unsatisfied, regardless of the spin configuration.
As another example, a square lattice in an antiferromagnetic setup might avoid frustration by allowing alternating spins. However, adding diagonals to such a square, forming triangles within the structure, can reintroduce frustration.

Note that frustration in the Ising model is topology-dependent \textit{and} constrained by the binary limitation of $K=2$ communities. 
The binary nature of spin states in the Ising model 
is relaxed  through the Potts model, as discussed in the sequel.


\subsubsection{The Potts Model}
\label{sec:potts}

The Potts model, proposed by Renfrey Potts in 1951, offers a crucial generalization by allowing each spin to take one of $K$ possible states, where $K$ can be as large as the number of nodes ($\nvertices$) in the network. This extension provides a more direct and flexible framework for community detection.
Reichardt and Bornholdt~\cite{RBER2006} formalized a Potts model specifically for this task. They established a set of first principles for a quality function, or Hamiltonian, designed to identify cohesive network communities. This Hamiltonian considers four fundamental types of pairwise relationships:

\begin{enumerate}
    \item \textbf{Rewarding internal links:} Edges connecting nodes within the same community should lower the system's energy. 
    \item \textbf{Penalizing internal non-links:} The absence of an edge between two nodes in the same community should increase the system's energy. 
    \item \textbf{Penalizing external links:} Edges connecting nodes in different communities should increase the system's energy. 
    \item \textbf{Rewarding external non-links:} The absence of an edge between nodes in different communities should lower the system's energy.
\end{enumerate}

This formulation represented a major conceptual advancement by explicitly accounting for both existing connections (links) and missing connections (non-links) within the objective function. Unlike the Ising model, which separates ferromagnetic and antiferromagnetic interactions into distinct regimes, the Potts model unifies them into a single quality function: rewarding link alignment (ferromagnetic) and penalizing non-link alignment (antiferromagnetic), thereby capturing both types of structural information simultaneously.

The simplified version of this model considers symmetric weights for links and non-links (i.e., the reward for an internal link equals the penalty for an external link, and the penalty for an internal non-link equals the reward for an external non-link). This simplification allows the Hamiltonian to be expressed purely in terms of internal interactions (rewarding internal links and penalizing internal non-links), thereby ignoring external relationships and yielding significant computational benefits by focusing the analysis on local structure. The resulting Hamiltonian takes the general form:
\begin{equation}\mathcal{H}_{\text{Potts}}(\partition) = - \sum_{i \neq j} (a_{ij}A_{ij} - b_{ij}(1 - A_{ij}))\delta(\sigma_i=\sigma_j)
    \label{eq:hpotts}
\end{equation}
where $A_{ij}$ is the adjacency matrix and the parameters $a_{ij}$ and $b_{ij}$ represent the weights for internal links and non-links, respectively.

\subsubsection{Modularity}
\label{sec:modularity}

From this simplified Potts framework, numerous community detection methods can be cast as special cases by defining the weights $a_{ij}$ and $b_{ij}$ in different ways. A prominent example is the widely used Modularity metric, introduced by Newman and Girvan in 2004~\cite{newman2004modularity}. 
Modularity can be derived from the simplified Hamiltonian by setting the weight $a_{ij} = 1 - b_{ij}$ and $b_{ij}$ (the penalty for internal non-links) to a specific null model, $p_{ij} = \frac{d_i d_j}{2m}$, where $d_i$ is the degree of node $i$ and $m$ the total number of edges in the network. Then, replacing the above values into~Eq.~\eqref{eq:hpotts}, we obtain:
\begin{equation}
    \mathcal{H}_{\text{Modularity}}(\partition) = - \sum_{i \neq j} \left ( A_{ij} - \frac{d_i d_j}{2m} \right ) \delta(\sigma_i=\sigma_j).
\end{equation}
However, this reliance on a global null model makes Modularity~\cite{newman2004modularity} susceptible to a significant drawback: the well-documented \textit{resolution limit problem}~\cite{traag2011cpm}.

\subsubsection{Resolution-limit problem}
\label{sec:resolution_limit}

This problem arises precisely because the null model term in $\mathcal{H}_{\text{Modularity}}$ depends on the degrees of the nodes ($d_i, d_j$) and the total number of edges in the entire graph ($m$), making the optimal partition dependent on the global scale of the network. This leads to inconsistent behavior. For instance, consider a classic example used to illustrate this problem: a ``ring of cliques'' network (as depicted in Figure 1 of~\cite{traag2011cpm}). In such a network, adjacent cliques in the ring are connected by a single link between specific nodes.

Modularity might find it optimal to merge pairs of cliques into single communities. Yet, if one of these merged pairs is extracted and analyzed as a separate subgraph, Modularity would then find it optimal to split them into two distinct communities (one for each clique). This inconsistency arises because the local decision is influenced by the global context.

A method is considered resolution-limit-free if, when analyzing any induced subgraph of the original graph, the partitioning results remain unchanged. Specifically, if $\partition=\{\community_{1},\community_{2},\dots,\community_{K}\}$ is an $\mathcal{H}$-optimal partition of a graph $\graph$, then the objective function $\mathcal{H}$ is called resolution-limit-free if for each subgraph $\graph'$ induced by $\mathcal{D}\subset \partition$, the partition $\mathcal{D}$ is also $\mathcal{H}$-optimal. This implies that a resolution-limit-free method will never depend on the size of the network to merge cliques in the ring of cliques network. If such a method merges cliques in a large graph, it must also merge them in any smaller subgraph containing those cliques, due to the resolution-limit-free property.

\subsubsection{Constant Potts Model}
\label{sec:cpm}

To address the resolution-limit problem, Traag et al. introduced the \textit{Constant Potts Model (CPM)}~\cite{traag2011cpm}. The ``constant" in CPM refers to replacing the traditional global, size-dependent null model with a multiplicative resolution term $\resolution$ that remains conditionally constant. Specifically, for unweighted graphs, the parameters $a_{ij}$ and $b_{ij}$ within the Potts Hamiltonian, $\mathcal{H}_{\text{Potts}}$ (Eq.~\ref{eq:hpotts}, where $a_{ij}$ weighs internal links and $b_{ij}$ weighs internal non-links), simplify to $a_{ij} = 1-\resolution$ and $b_{ij} = \resolution$. These values directly align with the internal-community pair values $v_{\anode\bnode}$ defined in Eq.~\eqref{eq:pair_value_detailed}, where $v_{\anode\bnode} = 1-\resolution$ for connected pairs and $v_{\anode\bnode} = -\resolution$ for non-connected pairs. 
The resulting Hamiltonian is:
\begin{equation}
    \mathcal{H}_{\text{CPM}}(\sigma) =
    - \sum_{i \neq j} ((1-\resolution)A_{ij} - \resolution(1 - A_{ij}))\delta(\sigma_i=\sigma_j) =
    - \sum_{i,j} (A_{ij} - \resolution)\delta(\sigma_i=\sigma_j) = -\sum_{i,j} v_{ij}.
    \label{eq:cpm}
\end{equation}
This formulation makes CPM a \textit{resolution-limit-free} method. The decision to merge or split communities depends only on the local density of connections relative to the resolution $\resolution$, ensuring consistent results regardless of the overall network size. A partition identified as a local minimum in a large network will remain a local minimum if analyzed as a subgraph.

Our work directly builds upon the  CPM. We do not propose a new model but rather introduce a novel analytical perspective by interpreting CPM through the lens of \textit{Hedonic Game Theory}~\cite{avrachenkov2017cooperative,brandt2016handbook}.

From this perspective, we contend that the term ``constant'' may not be the most appropriate, as it typically refers to a Potts Model with uniform multiplicative terms. However, the formulation considered here involves two distinct types of weights: a positive weight (ferromagnetic) and a negative weight (antiferromagnetic). These opposing forces can lead to what is known as \textit{frustration} in the Potts Model, resulting in a \textit{frustrated Potts Model} (see Section~\ref{sec:frustration}). This frustration introduces potential conflicts, which we interpret as a loss of robustness in game-theoretic terminology. We will further explore the connection between \textit{frustration} and \textit{robustness} in Section~\ref{sec:robustness}.

\paragraph{\textbf{Competitive Potts Model}}
While retaining the abbreviation CPM, we propose a more fitting name for the model under consideration: the \textit{Competitive Potts Model}. In this reinterpretation, nodes are rational agents that seek to maximize their individual utility by joining or forming communities. Therefore, this framework can either be considered  as a \emph{non-cooperative game} that models the emergence of cooperative structures~\cite{gonzalez2010introductory}, or as a \emph{cooperative game} to account for the fact that the ultimate interest relies on the derived coalitions~\cite{avrachenkov2017cooperative}.

\paragraph{\textbf{Non-transferable Utility}}
Non-cooperative game theory deals with individual players, their available strategies, and their resulting payoffs~\cite{gonzalez2010introductory}. Our model aligns   with this paradigm: nodes are the players, the chosen community is their strategy, and the local potential $\varphi_i^\resolution$ serves as their individual payoff function. Furthermore, this utility is \emph{non-transferable}, as there is no mechanism for agents to exchange or reallocate their utility gains; each agent's payoff is determined solely by the composition of its chosen coalition. In contrast, cooperative game theory abstracts away from such strategic details to focus on what coalitions can achieve collectively and how the joint benefits should be allocated~\cite{avrachenkov2017cooperative}.

\paragraph{\textbf{Potential Game}}
The approach considered in this paper follows a line of work in the literature that uses non-cooperative games as a foundation for cooperative outcomes~\cite{perez1994cooperative}.  
Agents do not make prior binding agreements; they act selfishly, and cooperation (the formation of stable communities) \emph{emerges as a Nash Equilibrium} of the game. Crucially, the local utility gains of individual agents are precisely aligned with the global improvement in a potential function (see Section~\ref{sec:potential_equivalence}). This structure defines the system as an \emph{exact potential game}, where self-interested actions converge toward equilibria that are beneficial for the collective.

\paragraph{\textbf{Competitive at Coalition and Intranode Levels}}
We use the term ``competitive'' because competition arises at two levels. First, at the coalition level, communities compete for the inclusion of nodes. Second, a conflict emerges at the intra-agent level, accounted for by the factor $A_{ij} - \resolution$, where each agent must balance two competing preferences: maximizing internal links (ferromagnetic cohesion) and minimizing internal non-links (antiferromagnetic exclusivity). This dual nature of emergent cooperation and underlying competition frames CPM not only as a potential game but also as a meaningful abstraction of real-world community dynamics. This new concept allows us to trace a direct theoretical lineage from the Ising model in physics to a modern, game-theoretic interpretation of community detection, a connection further summarized in Table~\ref{tab:game-ising-community}.

\begin{table}[t]
\centering
\caption{Ising or Potts Model (Physics), Constant Potts Model (Community Detection), and Hedonic Games (Game Theory)}
\scriptsize
\renewcommand{\arraystretch}{1.3}
\begin{tabular}{|p{5cm}|p{5cm}|p{5cm}|}
\hline
\textbf{Ising or Potts Model (Physics)} & \textbf{Community Detection} & \textbf{Hedonic Game Theory} \\
\hline \hline
\textbf{Spins} (magnetic particles) & \textbf{Nodes}  in a network &
\textbf{Players}  (agents) \\
\hline
{Aligned spins} (same state) & {Community}  (cluster of nodes) &
{Coalition}  (group of agents) \\
\hline
{Local energy} & {Tradeoff}:\newline internal links vs. internal non-links &
{Individual utility},\newline $\varphi_i^\resolution = (1 - \resolution)d_i^k - \resolution \hat{d}_i^k$ \\
\hline
{Ground state}\newline (no local change reduces energy) & {Locally optimal partition} &
{Nash Equilibrium}\newline (no agent can improve utility by moving) \\
\hline
\textbf{Frustration}:\newline energy cannot be minimized for all edges & \textbf{Conflicting objectives}:\newline best community differs by criterion &
\textbf{Frustration}:\newline No coalition satisfies all preferences \\
\hline
\textbf{Hamiltonian}:\newline global energy decreases with better alignment & \textbf{Constant Potts Model:}\newline minimizes internal non-links and maximizes links &
\textbf{Potential Game}:\newline global potential increases with utility \\
\hline
\textbf{Spin with satisfied interactions}:\newline aligned with all neighbors & \textbf{Robust node}:\newline community is optimal across $\resolution \in [0,1]$ & \textbf{Totally robust agent}:\newline best group for all $\resolution$ \\
\hline
{Spin flips}  (e.g., Glauber dynamics) & {Node relocation}  (e.g., Leiden algorithm) &
{Best-response dynamics} \\
\hline
{Low-energy configuration} & {Robust and stable}  community structure &
{Stable partition} means strategic equilibrium \\
\hline
\end{tabular}
\label{tab:game-ising-community}
\end{table}

\subsection{From the Constant Potts Model to the Competitive Potts Model through Hedonic Games}
\label{sec:from_cpm_to_hedonic}


Despite extensive research in both hedonic games and community detection, the connection between community detection algorithms and potential hedonic games remains underexplored. In this work, we relate the CPM objective function~\cite{traag2011cpm} with a hedonic potential~\cite{avrachenkov2017cooperative, brandt2016handbook}.
The local improvement Algorithm~\ref{alg:simple-better-response}, which solves a hedonic game when considering CPM as its quality function, shares conceptual foundations with earlier community detection methods that iteratively refine a quality function.

Hedonic games, extensively studied in the literature~\cite{dreze1980hedonic, banerjee2001core, cechlarova2001stability, bogomolnaia2002stability, brandt2016handbook}, define a setting where each node has a utility for belonging to a given community, known as its \emph{hedonic potential}, which quantifies its fit within the community.
From a computational standpoint, potential hedonic games have utility functions that can be computed in polynomial time~\cite{gairing2010computing}, relying only on the number of neighbors of a node and the size of each community. From a social network perspective, hedonic games provide a natural model for representing the constraints on forming and maintaining social connections~\cite{sheikholeslami2016egonet}.

To determine whether a node has an incentive to deviate from its current community, we employ a multi-scale game-theoretic framework that integrates local (node-level), meso (community-level), and global (partition-level) perspectives. 
Derived from Eq.~\eqref{eq:pair_value} to Eq.~\eqref{eq:cpm}, our potential functions formalize the core objectives of maximizing ``\textit{friends}'' and minimizing ``\textit{strangers}'' across scales, ranging from individual to societal levels, as further detailed in the sequel.

\subsubsection{Hedonic Potential}
\label{sec:hedonic_cpm}

The objective function is used to assign scores to each node, community, and partition, reflecting their structural quality. The objective function is referred to as a \textit{quality function} in community detection, a \textit{potential function} in game theory, and a \textit{Hamiltonian} in physics that measures the total energy of a system to be minimized.
In what follows, we demonstrate that optimizing the  CPM is equivalent to computing a stable partition of a  potential hedonic game with Additively Separable Preferences (ASP)~\cite{avrachenkov2017cooperative}. 
The node potential $\potentialNode_{\anode}^{\resolution}$ evaluates the contribution of a single node to a community, the community potential $\potentialCoalition^{\resolution}$  aggregates the contributions of all nodes within a community, and the partition potential $\potentialPartition^{\resolution}$  sums the contributions across all communities in a partition. While they share a common mathematical structure, their distinct scopes reflect the hierarchical nature of community detection, from individual nodes to global network partitions,
noting that there is an equivalence between node and partition gains (see~Section~\ref{sec:potential_equivalence}).


\paragraph{\textbf{Node Potential}} \label{sec:node_potential}

The \textit{potential of node} $\anode$ with respect to  community $\community_k$ is the sum of $v_{\anode\bnode}$ for all $\bnode \in \community_k$:

\begin{equation}
    \potentialNode_{\anode}^{\resolution}(\community_k)=
    \sum_{\bnode \in \community_k} v_{\anode\bnode}.
\end{equation}
The above potential can be written in terms of the node degrees as follows,
\begin{equation}
      \potentialNode_{\anode}^{\resolution}(\community_k)=  (1 - \resolution) \degrik - \resolution \cdegrik
    \label{eq:potential_node}
\end{equation}
where $\degrik$ is the degree of node $\anode$ in community $\community_k$ and  $\cdegrik$ is the number of non-neighbors of node $\anode$ in community~$\community_k$. As discussed in the introduction, we consider a \textit{max-min} problem of maximizing the number of neighbors and minimizing non-neighbors within communities.   The above expression for the node potential, to be maximized by each node, balances between these two objectives, of   maximizing neighbors within community ($\degrik$) and minimizing non-neighbors within community ($\cdegrik$), using the resolution parameter $\resolution$ to trade between the two.

Let $n_k$ be the number of nodes in community $\community_k$. Recall that $\delta(\sigma_{\anode}= k)$  is  an indicator function that equals 1 if $\sigma_i=k$, and 0 otherwise.
Then, \begin{equation}\hat{d}^k_{\anode}=n_k-d_{\anode}^k-\delta(\sigma_{\anode}= k).\label{eq:dik} \end{equation} Therefore, replacing~\eqref{eq:dik} into~\eqref{eq:potential_node} we obtain:
\begin{equation}   \potentialNode_{\anode}^{\resolution}(\community_k)= 
    \degrik - \resolution (\degrik + \cdegrik) =
    \degrik - \resolution (n_k - \delta(\sigma_{\anode}=k)). \label{eq:potential_node1}
\end{equation}
In what follows, we relate the node potential to the community potential.

\paragraph{\textbf{Community Potential}} \label{sec:community_potential}

The \textit{potential of community} $\community_k$ is half the sum of the potential of all nodes $\anode \in \community_k$, given by:
\begin{equation}    \potentialCoalition^{\resolution}(\community_k)=
    \frac{1}{2} \sum_{\anode\in\community_k} \potentialNode_{\anode}^{\resolution}(\community_k)=
    \frac{1}{2} \sum_{\anode\in\community_k} \sum_{\bnode\in\community_k} v_{\anode\bnode}.
    \label{eq:potential_community}
\end{equation}
Let   $m_k$ be the number of edges in  community $\community_k$. 
Note that 
\begin{equation}
    \sum_{\anode \in \mathcal{S}_k} d_{\anode}^k = 2 m_k. \label{eq:standardgraph}
\end{equation}
Then, replacing~\eqref{eq:standardgraph} and~\eqref{eq:potential_node1}  into~\eqref{eq:potential_community} we obtain
\begin{equation}  \potentialCoalition^{\resolution}(\community_k)=  m_k - \resolution \frac{n_k (n_k-1)}{2}  =
   m_k - \resolution \binom{n_k}{2} = (1-\resolution) m_k - \resolution \left(\binom{n_k}{2}- m_k\right).   \label{eq:neighnonnei}
\end{equation}
As previously discussed, the above potential is motivated by a \textit{max-min} problem consisting of maximizing  the number of neighbors, $m_k$, and minimizing the number of  non-neighbors,  $\binom{n_k}{2}- m_k$, within communities.  The resolution parameter $\resolution$ balances the two goals.

\paragraph{\textbf{Partition Potential}} \label{sec:partition_potential} 
The \emph{potential of  partition} $\pi$ is given by:
%
\begin{equation}
    \potentialPartition^{\resolution}(\partition)=
    \sum_{k=1}^K \potentialCoalition^{\resolution}(\community_k)=
    \sum_{k=1}^K \sum_{\anode\in\community_k}  \sum_{\bnode\in\community_k} v_{\anode \bnode}= \sum_{k=1}^K 
   \sum_{\anode\in\community_k} \sum_{\bnode\in\community_k} \left(A_{\anode\bnode} - \resolution \right)\delta(\sigma_{\anode}= \sigma_{\bnode}).
     \label{eq:potential_partition}
\end{equation}
Contrasting the above equation against Eq.~\eqref{eq:pair_value}, we note that the potential of a partition equals the sum of the values of all its node pairs. Then, it follows from Eqs.~\eqref{eq:neighnonnei} and~\eqref{eq:potential_partition} that \begin{equation}  \potentialPartition^{\resolution}(\partition)= (1-\resolution) \left( \sum_{k=1}^K m_k \right) - \resolution \left( \sum_{k=1}^K\binom{n_k}{2}- m_k\right),   \label{eq:neighnonnei1}
\end{equation}
where the first term corresponds to the number of neighbors, and the second to the number of non-neighbors within communities.

\subsubsection{Node Gain Equals Partition Gain}
\label{sec:potential_equivalence}


We demonstrate a fundamental property of the considered hedonic game: gains from a single node moving between communities are identical whether measured by the node's potential change ($\Delta\potentialNode$) or the overall partition's potential change ($\Delta\potentialPartition$). This property
ensures that selfish moves align with global improvement.

%
%

Consider node $\anode$ moving from its origin community $\clusterB$ to destination community $\clusterA$.  In this section, to simplify terminology we consider two communities, i.e.,  $\partition = \{ \clusterA, \clusterB  \}$, noting that our results hold for any $K \ge 1$   as the two communities of interest can be arbitrarily chosen.   Let  $
\partition' = \{\clusterA \cup \{\anode\},\, \clusterB \setminus \{\anode\}\}
$
denote the updated partition after this move.

We denote by   $\Delta \potentialNode_{\anode}^{\clusterB \clusterA}$
the node potential gain  of  $\anode$ due to its move, and by  $ \Delta \potentialPartition_{\anode}^{\clusterB \clusterA}$ the corresponding   partition potential gain,
\begin{align}
    \Delta \potentialNode_{\anode}^{\clusterB \clusterA} =    \potentialNode_{\anode}^{\resolution}(\clusterA \cup \{\anode\}) - \potentialNode_{\anode}^{\resolution}(\clusterB),  \quad  \Delta \potentialPartition_{\anode}^{\clusterB \clusterA}
    =\potentialPartition^{\resolution}(\partition') - \potentialPartition^{\resolution}(\partition).
\end{align}
Next, we show that 
\begin{equation}
     \Delta \potentialNode_{\anode}^{\clusterB \clusterA} = \Delta \potentialPartition_{\anode}^{\clusterB \clusterA}. \label{eq:nodeeqpartition}
\end{equation}
Indeed, it follows from Eq.~\eqref{eq:potential_node1} that
\begin{align}
    \Delta \potentialNode_{\anode}^{\clusterB \clusterA}
    = \potentialNode_{\anode}^{\resolution}(\clusterA \cup \{\anode\}) - \potentialNode_{\anode}^{\resolution}(\clusterB)
    = \left(d_{\anode}^{\clusterA} - \resolution n_{\clusterA}\right) - \left(d_{\anode}^{\clusterB} - \resolution (n_{\clusterB} - 1)\right)
    = \left(d_{\anode}^{\clusterA} - d_{\anode}^{\clusterB}\right) - \resolution\left(n_{\clusterA} - n_{\clusterB} + 1\right).
    \label{eq:delta_node}
\end{align}
Similarly, it follows from Eq.~\eqref{eq:potential_partition} that 
\begin{equation}
    \Delta \potentialPartition_{\anode}^{\clusterB \clusterA}
    =\potentialPartition^{\resolution}(\partition') - \potentialPartition^{\resolution}(\partition) 
    =\potentialCoalition^{\resolution}(\clusterA \cup \{\anode\}) + \potentialCoalition^{\resolution}(\clusterB \setminus \{\anode\}) 
    -\left(\potentialCoalition^{\resolution}(\clusterA) + \potentialCoalition^{\resolution}(\clusterB)\right).
    \label{eq:partition_gain_initial}
\end{equation}
Now, substituting  Eq.~\eqref{eq:neighnonnei} into Eq.~\eqref{eq:partition_gain_initial},
\begin{align}
 \Delta \potentialPartition_{\anode}^{\clusterB \clusterA}
    &= (m_{\clusterA} + d_{\anode}^{\clusterA}) - \resolution \binom{n_{\clusterA}+1}{2}  + (m_{\clusterB} - d_{\anode}^{\clusterB}) - \resolution \binom{n_{\clusterB}-1}{2}-\left(m_{\clusterA} - \resolution \binom{n_{\clusterA}}{2} + m_{\clusterB} - \resolution \binom{n_{\clusterB}}{2}\right) \\
    &=
    \left(d_{\anode}^{\clusterA} - d_{\anode}^{\clusterB}\right) - \resolution\left(n_{\clusterA} - n_{\clusterB} + 1\right) =  \Delta \potentialNode_{\anode}^{\clusterB \clusterA}.
    \label{eq:delta_partition}
\end{align}

The above property plays a key role in establishing the convergence of the considered community detection algorithm (see Theorem~\ref{THEO:EQUILIBRIUM}). Indeed, by deriving a lower bound on $ \Delta \potentialNode_{\anode}^{\clusterB \clusterA}$, we show that  greedy 
local moves lead to a stable, locally optimal partition, in pseudo-polynomial time.

\subsubsection{Hedonic Game Equilibrium}
\label{sec:hedonic-leiden}

Having established the potential functions that guide individual node decisions, we now define the collective outcome of these choices. The central question is: when does the community formation process reach a stable state? In our hedonic game framework, this state is known as an \textbf{equilibrium}, i.e., a partition where no self-interested node has an incentive to move to another community. This section formally defines the concepts of stability and equilibrium, which serve as candidate solutions to the community detection problem.

\begin{definition}[Stability Concepts]
A community is:
\begin{itemize}
    \item \textbf{Internally stable} if no node within the community has an incentive to leave it for any other existing community.
    \item \textbf{Externally stable} if no node outside the community has an incentive to join it.
\end{itemize}
A partition is \textbf{stable} if all its communities are both internally and externally stable. Such a stable partition is referred to as an \textbf{equilibrium}.
\end{definition}

From the definitions above, we formally establish that a hedonic game equilibrium serves as a candidate solution to the community detection problem, as no node has an incentive to change its assigned community.

\begin{definition}[Local Stability~\cite{traag2019leiden}] 
A node $\anode$ is \textbf{locally stable} in partition $\partition$ if its utility in its current community $\community_{\sigma_i}$ is greater than or equal to its utility if it deviates to another community:
\[ \potentialNode_{\anode}^{\resolution}(\community_{\sigma_i}) \geq \potentialNode_{\anode}^{\resolution}(\community_k) \quad \forall k \neq \sigma_i. \]
\end{definition}

\begin{definition}[Solution~\cite{bogomolnaia2002stability}] \label{def:solution}
A \textbf{solution} to the community detection problem is an equilibrium partition, i.e., a partition wherein all nodes are locally stable.
\end{definition}

In an equilibrium partition $\partition^\resolution$, for each node $\anode$, its assigned community $\sigma_{\anode}$ must be in the set of communities that maximize its utility:
\[ \sigma_{\anode} \in \arg\max_k \potentialNode_{\anode}^{\resolution}(\community_k). \]
A node only moves if a new community offers strictly higher utility. In the case of ties for the maximum utility, a node has no incentive to move from its current community if it is one of the optimal choices. Note that the definitions of local stability and equilibrium subsume a fixed $\resolution$. In the remainder of this work, we refer to an equilibrium as a $\resolution$-equilibrium whenever we aim to stress the dependence of the equilibrium on $\resolution$.
Whenever we aim to stress that the equilibrium holds for a range $[\resolution_0, \resolution_1]$ of values of $\resolution$, we refer to it as a  $[\resolution_0, \resolution_1]$-equilibrium.  

As a direct consequence of the equivalence between node and partition potential gains (see Section~\ref{sec:potential_equivalence}), a solution can also be characterized via the global partition potential~\cite{monderer1996potential}.

\begin{proposition}[Potential Maximized at Equilibrium~\cite{monderer1996potential}]\label{coro:pot}
A solution to the community detection problem is a partition wherein no node can unilaterally move to another community and increase the global partition potential.
\end{proposition}

\begin{proof}
The proof follows directly from the property that $\Delta \potentialNode_{\anode} = \Delta \potentialPartition_{\anode}$ (see Eq.~\eqref{eq:delta_partition}). If all nodes are locally stable, no node can unilaterally move to increase its own potential, i.e., we have $\Delta \potentialNode_{\anode} \le 0$  for every unilateral move of any node $i \in V$. Therefore, no unilateral move can increase the partition's potential given by Eq.~\eqref{eq:potential_partition}.
\end{proof}


Next, we discuss the role of $\resolution$ in the convergence time towards an equilibrium, which Algorithm~\ref{alg:simple-better-response} provides as a simple better-response dynamics to find.

\section{Analytical Contributions}
\label{sec:analytical_contributions}

In Section~\ref{sec:equilibrium_convergence}, we demonstrate that if the resolution parameter $\resolution$ is a rational number, an iterative better-response dynamics (e.g., Algorithm~\ref{alg:simple-better-response}) is guaranteed to converge in pseudo-polynomial time to an equilibrium partition according to Definition~\ref{def:solution}. This convergence is ensured because the partition potential, $\potentialPartition$, is bounded within $[-\nvertices^2, \nvertices^2]$ and strictly increases with each step. Specifically, for a rational $\resolution = \numer/\denom$, each beneficial node move increases the potential by at least $1/\denom$, guaranteeing convergence in $O(\denom \nvertices^2)$ steps (Theorem~\ref{THEO:EQUILIBRIUM}).
Then, in Section~\ref{sec:robustness}, we shift our focus from convergence guarantees to the structural stability of the resulting equilibria, introducing formal criteria to evaluate the robustness of a partition against variations in the resolution parameter $\resolution$.

\subsection{Equilibrium Convergence}
\label{sec:equilibrium_convergence}

Our analysis of equilibrium convergence explores key aspects:

\begin{itemize}
    \item 
\textbf{Finding an Equilibrium in Pseudo-Polynomial Time}: In  Section~\ref{sec:find_equilibrium} we prove that the better-response dynamics described in Algorithm~\ref{alg:simple-better-response} finds an equilibrium partition in $O(\denom \nvertices^2)$ time. The proof hinges on the fact that the potential gain from any single node move is lower-bounded, ensuring a finite number of steps before reaching a stable state where no node has an incentive to move.
In Section~\ref{sec:fixed_vs_variable_resolution}  we  consider extremal values for $\resolution$ and show that in those cases the algorithm converges in polynomial time.

\item 
\textbf{Best- vs. Better-Response Dynamics}: In Section~\ref{sec:best_vs_better}, we show that neither best-response dynamics (always choosing the move with the maximum potential gain) nor better-response dynamics (choosing any move with a positive gain) is universally faster. We provide examples demonstrating the existence of specific sample paths where a better-response dynamics converges to an equilibrium in fewer steps than a globally greedy best-response dynamics.
\end{itemize} 

 In Section~\ref{sec:summary_convergence} we summarize our convergence results, indicating that while pseudo-polynomial convergence is guaranteed, the practical performance depends on design choices such as the resolution parameter $\resolution$, the initial partition, and the node selection strategy. 

\subsubsection{Finding an Equilibrium Partition}
\label{sec:find_equilibrium}


\paragraph{\textbf{Response Dynamics: Which Node? Which Community?}}
Two key questions arise in the dynamics to find an equilibrium partition:  
(a)   Which node should move next? (b)  Given a node, to which community should it move?  
For (a), the algorithm may use a  best-response (select a  node   that maximizes the gain in potential) or a better-response (allow any node with a strictly positive gain to move).
For (b), a node may adopt a local  best-response  (move to the community with maximum potential, Eq.~\ref{eq:potential_node}) or a local better-response (move to any community with higher potential than its current one).  

\paragraph{\textbf{Reference Algorithm}}
Algorithm~\ref{alg:simple-better-response} captures a better-response dynamic when  determining which node moves next. For a selected node $\anode$ (line~\ref{for:node}), it identifies a set of communities $\mathcal{K}^{\star}$ that could improve its potential. This set may be restricted to the communities that maximize its local potential, defining a local best-response (line~\ref{best:communities}), or it may include any community with higher potential than the current one, corresponding to a local better-response (line~\ref{better:communities}). A target community $\community_{\kappa}$ is then selected (line~\ref{let:kappa}), and if it yields strictly greater potential (line~\ref{if:increases}), the node is moved (line~\ref{set:move}). The process continues until no node can further improve its potential, reaching the equilibrium of Definition~\ref{def:solution}.

Our theoretical results apply irrespective of whether nodes adopt best- or better-response strategies (see Section~\ref{sec:best_vs_better}). In simulations (Section~\ref{sec:community_tracking}), we implement Algorithm~\ref{alg:simple-better-response} using the local best-response at the node level (line~\ref{best:communities}). To improve efficiency, our implementation follows the Leiden algorithm~\cite{traag2019leiden}, which maintains a queue of candidate nodes $\mathcal{Q}$ (line~\ref{init:queue}). Whenever a node moves, its neighbors outside the target community are added to the queue (line~\ref{set:queue}). This prioritizes nodes whose neighborhoods have recently changed, focusing computation on active regions of the network~\cite{blondel2024louvain}. 

\begin{algorithm}[t]
\SetKwInOut{Input}{Input}
\SetKwInOut{Output}{Output}
\Input{
    $\graph=(\vertices,\edges)$;
    $\resolution$;
    Initial partition membership~$\sigma: \vertices \to \{1, \ldots, K\}$.
}
\Output{
    Equilibrium partition $\sigma$.
}
$\texttt{changed} \leftarrow \texttt{true}$ \label{let:flag} \tcp*{Ensures the main loop runs at least once}
\While{$\texttt{changed}$ \label{while:changed}}{
    $\texttt{changed} \leftarrow \texttt{false}$ \label{set:flag_false} 
    
    Initialize queue $\mathcal{Q} \leftarrow \vertices$ \label{init:queue} 
    
    \While{$\mathcal{Q}$ is not empty \label{while:queue}}{
        Remove a node $i$ from $\mathcal{Q}$ \label{for:node} \tcp*{Select a candidate node to evaluate}
        
        \uIf{Node-level best-response \label{if:best}}{
             $\mathcal{K}^{\star} \leftarrow \arg\max_{k \in \{1, \ldots, K\}} \varphi_i^\resolution(\community_k)$ \label{best:communities} \tcp*{Best communities for node $i$}
        }
        \Else{
             $\mathcal{K}^{\star} \leftarrow \{ {k \in \{1,\ldots,K\}} \text{ such that } \varphi_i^\resolution(\community_k) > \varphi_i^\resolution(\community_{\sigma_i})\}$ \label{better:communities} \tcp*{Better communities for node $i$}
        }

        Let $\kappa$ be an element from $\mathcal{K}^{\star}$ \label{let:kappa} \tcp*{Select one target community (handling ties)}
        
        \If{$\varphi_i^\resolution(\community_\kappa) > \varphi_i^\resolution(\community_{\sigma_i})$ \label{if:increases}}{
            $\sigma_i \leftarrow \kappa$ \label{set:move} 
            
            $\texttt{changed} \leftarrow \texttt{true}$ \label{set:flag}  

            $\mathcal{Q} \leftarrow  \mathcal{Q} \cup \{\bnode \in \vertices \mid A_{\anode \bnode}=1 \textrm{ and }  \bnode \notin \community_{\kappa} \}$ \label{set:queue} \tcp*{Leiden's local move queue management~\cite{traag2019leiden}}
        }
    }
}
\Return{$\sigma$} \tcp*{Return partition as solution according to Definition~\ref{def:solution}}
\caption{\textbf{Better-Response Dynamics for Finding a Hedonic Game Equilibrium.}}
\label{alg:simple-better-response}
\end{algorithm}

\paragraph{\textbf{Convergence}}
Next, we show that the resulting  equilibrium partition produced by Algorithm~\ref{alg:simple-better-response}
is identified in pseudo-polynomial time using better-response dynamics.
To this aim, we assume that $\resolution$ is rational.%
\footnote{The convergence properties of the dynamics for irrational values of $\resolution$ are not covered by this proof and remain an open question.}
i.e., $\resolution=\numer/\denom$, where $\numer, \denom \in \mathbb{N}$, and $b \leq c$.
Under this   condition, the following theorem applies: 
\begin{theorem} \label{THEO:EQUILIBRIUM}
Community detection based on hedonic games finds an equilibrium partition in $O(\denom \nvertices^2)$ where $\resolution=\numer/\denom$, and  $\numer,\denom \in \mathbb{N}$.
\end{theorem} 

The above theorem  establishes an asymptotic upper bound on the convergence time. A natural lower bound is determined by the time required to read the graph: \( \Omega(\nvertices^2) \) when using an adjacency matrix representation and \( \Omega(\nvertices+\medges) \) when using  an adjacency list.

Before examining the proof of Theorem~\ref{THEO:EQUILIBRIUM}, we briefly outline the basic insights involved. As pointed out above,   the potential $\potentialPartition$ satisfies $-n^2 \leq \potentialPartition \leq \nvertices^2$. Additionally, the potential can be written as:
\begin{align}
    \potentialPartition = \frac{U}{\denom}
    \label{eq:potential_partition_frac}
\end{align}
where $U$, ${\denom} \in \mathbb{N}$. Consider a node $\anode$ that can gain from moving between communities. The key step in the proof is showing that such a move yields an instantaneous gain lower bound by $1/\denom$. This lower bound, together with the maximum and minimum values attained by $\potentialPartition$, implies that the algorithm converges in at most $2{\denom} \nvertices^2$ steps. The worst-case scenario occurs if, at the first iteration of the algorithm, $U=-{\denom}\nvertices^2$ and at the last iteration, $U={\denom}\nvertices^2$.

Note that  $c$ is a natural number, that can be encoded with $\log c$ bits. As computational complexity measures difficulty with respect to the length of the encoded input, the complexity of the considered algorithm  is  pseudo-polynomial, as the algorithm takes $c$ steps for an input of  length $\log c$.  Alternatively, if the input is represented in a non standard fashion, e.g., in  unary base, the algorithmic complexity would be referred to as polynomial as opposed to pseudo-polynomial.

\begin{proof}

We consider, without loss of generality, a partition $\partition $  where a tagged node $\anode$ can gain from moving from community $\clusterB$ to community $\clusterA$. Our goal is to show that the instantaneous gain derived from such a move is at least $1/\denom$. This result, together with the fact that the potential is lower bounded by $-\nvertices^2$ and upper bounded by $\nvertices^2$, implies that the algorithm converges in at most $2 \denom \nvertices^2$ steps.

Recall that $n_\clusterB$ is the number of nodes in the community $\clusterB$ \textit{before} the move (including $\anode$), while $n_\clusterA$ is the number of nodes in $\clusterA$, also \textit{before} the move (hence, not including  $\anode$).
Also, $\degriB$ and $\degriA$ are the degrees of $\anode$ in $\clusterB$ and $\clusterA$ before and after the move, and $\cdegrik$ is the number of nodes in community $\community_k$ not connected to $\anode$, with $\cdegriA = n_\clusterA - \degriA$ and $\cdegriB = n_\clusterB - \degriB - 1$. Therefore, $n = \degriA + \cdegriA + \degriB + \cdegriB + 1 = n_\clusterA + n_\clusterB$.

It follows from~\eqref{eq:delta_node} that $\denom \Delta \potentialNode_{\anode}^{\clusterB\clusterA} \in \mathbb{Z}$,
\begin{align}
    \denom \Delta \potentialNode_{\anode}^{\clusterB\clusterA} &= \denom (\degriA - \degriB) - \numer (n_\clusterA - n_\clusterB + 1). \label{eq:dv}
\end{align}
Then,
\begin{align}
    \Delta \potentialNode_{\anode}^{\clusterB\clusterA} > 0
    \Rightarrow \denom \Delta \potentialNode_{\anode}^{\clusterB\clusterA} > 0
    \Rightarrow  \denom \Delta \potentialNode_{\anode}^{\clusterB\clusterA} \ge 1
    &\Rightarrow \Delta \potentialNode_{\anode}^{\clusterB\clusterA} \ge \frac{1}{\denom},
    \label{eq:pass2}
\end{align}
where the first strict inequality follows from the condition for a move in Algorithm~\ref{alg:simple-better-response} (line~\ref{if:increases}), which requires a strictly positive gain in potential.
The second implication follows from the fact that $\denom \Delta \potentialNode_{\anode}^{\clusterB\clusterA} \in \mathbb{Z}$  (see Eq.~\eqref{eq:dv} and recall that we assume   that $\resolution$ is rational).  As communities  $\clusterA$ and $\clusterB$ were chosen arbitrarily, the argument holds for any such pair of communities. Thus, the proof is complete.

\end{proof}

\subsubsection{Extremal Resolution Parameters}
\label{sec:fixed_vs_variable_resolution}

In this subsection we  consider extremal values for the resolution parameter, $\resolution=0$ or $\resolution=1$. Before doing so, we begin by distinguishing between two cases: one wherein the resolution parameter is fixed a priori, and another where it is variable and given as input to the problem. 
\begin{definition}[Fixed vs. Variable Resolution Parameter]
We say that the resolution parameter $\resolution$ is:
\begin{itemize}
    \item \textit{Fixed}, if its value is externally given and does not form part of the input;
    \item \textit{Variable}, if its value is provided as part of the input. In this case,   we assume that $\resolution = b/c$, where $b, c \in \mathbb{N}$ and $b \leq c$.
\end{itemize}
\end{definition}

\begin{theorem}[Convergence for Extremal  Resolution Parameters]
Let $\resolution \in \{0,1\}$ be the resolution parameter used in the hedonic game-based community detection algorithm. Then the algorithm converges to an equilibrium partition in at most $\nvertices^2$ steps, where $\nvertices$ is the number of nodes in the graph. \label{theo:extremal}
\end{theorem}

\begin{proof}
\textbf{Case $\resolution = 0$.} In this case, the node potential simplifies to
\[
\potentialNode_{\anode}^{0}(\community_k) = d_i^{k},
\]
which is an integer. Each improving move increases the global potential by at least 1, and since in this case  the  partition potential is bounded above by $\nvertices^2$ and below by 0, convergence occurs in at most $\nvertices^2$ steps.

\textbf{Case $\resolution = 1$.} The node potential becomes
\[
\potentialNode_{\anode}^{1}(\community_k) = -\hat{d}_i^{k}  = d_i^{k} - \left(n_{k} - \delta(\sigma_i = k)\right),
\]
which is also an integer. Again, each improving move strictly increases the global potential by at least 1.  In this case,  the partition potential is bounded below by $-\nvertices^2$ and above by 0, leading to convergence in at most $\nvertices^2$ steps.
\end{proof}

More generally, for any fixed rational resolution parameter $\resolution = b/c$ with constant $c$, the algorithm still converges in $O(\nvertices^2)$ steps.
 In contrast, if $\resolution = b/c$ is provided as part of the input (i.e., variable $\resolution$), then $c$ may grow with input size, and convergence time becomes $O(c \nvertices^2)$, which is pseudo-polynomial with respect to $\log c$.

\subsubsection{Best- vs. Better-Response Dynamics}
\label{sec:best_vs_better}

In this section, we compare best-response and better-response dynamics. 
Here, a best-response dynamic means always selecting the node that maximizes its potential gain (line~\ref{for:node} of Algorithm~\ref{alg:simple-better-response}) and then applying a node-level best response (line~\ref{best:communities}). 
At first glance, one might expect this globally greedy strategy to converge the fastest. 
However, our analysis shows that this is not always the case: there exist instances where the simpler better-response dynamic reaches equilibrium in fewer steps, demonstrating that the best-response approach is not necessarily the most efficient path to convergence.

\begin{theorem}[Comparison of Convergence: Best-Response Dynamics vs. Better-Response Dynamics]
In hedonic games with potential-based dynamics, the convergence time under
better-response dynamics
may be equal to, shorter than, or longer than the convergence time under best-response dynamics. Specifically:
\begin{itemize}
    \item[(i)] There exist instances where both better- and best-response dynamics converge in the same number of steps.
    \item[(ii)] There exist instances where best-response dynamics converge faster than better-response dynamics.
    \item[(iii)] There exist instances where better-response dynamics converge faster than best-response dynamics.
\end{itemize} \label{theo:bestbetter}
\end{theorem}

\begin{proof}
We demonstrate each case with a concrete example.

\textbf{(i) Same number of steps.} Consider a graph with two connected nodes and $\resolution = 0$. Let each node initially belong to its own community. Since the potential is given by the intra-community degree, only one node has an incentive to move, and doing so leads directly to the equilibrium. In this case, both better- and best-response dynamics perform the same single move and converge in one step.

\textbf{(ii) Best-response dynamics is faster}.
Consider a complete graph on $\nvertices$ nodes (a clique), and let $\resolution = 0$. Initially, each node is in a singleton community. Under best-response dynamics, at each step we select a node that maximizes the increase in potential,  joining the largest existing coalition. This causes the largest community to grow fastest, yielding convergence in $O(\nvertices)$ steps. However, under better-response dynamics, nodes may join suboptimal communities, leading to more fragmented growth and potentially $O(\nvertices \log \nvertices)$ convergence. This is because in each round, multiple small coalitions compete, and the number of rounds needed to merge them all increases logarithmically.


\textbf{(iii) Better-response dynamics is faster.}
Figure~\ref{fig:graphexamples} illustrates an instance where it is possible for better-response dynamics to converge faster than best-response dynamics, highlighting that a greedy approach is not always the most efficient path to equilibrium on a step-by-step basis.%
\footnote{Our analysis here focuses on demonstrating the existence of specific sample paths where a better-response dynamic converges in fewer steps than a globally greedy best-response one. A formal comparison of the average-case complexity of these dynamics is a non-trivial task that we leave as a subject for future work.}
Table~\ref{tab:betterbest} presents four instances, each with a different resolution parameter $\resolution$, showing initial configurations where the number of iterations until convergence is strictly smaller under better-response dynamics. This phenomenon arises because best-response dynamics is greedy: it selects the move with the highest immediate utility gain, potentially ignoring promising paths that would require intermediate, less rewarding steps. In contrast, better-response dynamics allow for smaller, yet improving steps that can ultimately lead to faster convergence.

In Cases 1 and 2, corresponding to $\resolution = 0.1$ and $\resolution = 0.25$, both best- and better-response dynamics converge to the same equilibrium: the grand coalition $\{0,1,2,3\}$. When $\resolution = 0.1$, the total utility gain from the initial partition to the grand coalition equals 2.5; when $\resolution = 0.25$, it equals 1.75.  Under better-response dynamics, the gains are divided as $0.8 + 1.7$ and $0.5 + 1.25$, respectively. In general, these correspond to $1-2\resolution$ and $2(1-\resolution)-\resolution=2-3\resolution$, respectively. Under best-response dynamics, they are divided as $0.9 + 0.9 + 0.7$ and $0.75 + 0.75 + 0.25$, respectively.   In general, these correspond to $1-\resolution$, $1-\resolution$ and $1-3\resolution$, respectively (see Figure~\ref{fig:examplegraph}). 
\end{proof}


 
\begin{figure}[h!]
\centering
\begin{minipage}{0.45\textwidth}
\centering
\begin{tikzpicture}[every node/.style={circle,draw,fill=light_gray,minimum size=6mm}, node distance=2cm]
\node (0) at (90:2cm) {0};
\node (1) at (0:2cm) {1};
\node (2) at (270:2cm) {2};
\node (3) at (180:2cm) {3};

\draw (0) -- (1);
\draw (0) -- (2);
\draw (0) -- (3);
\draw (1) -- (2);
\end{tikzpicture}
\caption{Graph used in examples showing better-response dynamics may converge faster than best-response dynamics.}
\label{fig:graphexamples}
\end{minipage}
\hfill
\begin{minipage}{0.5\textwidth}
\centering
\captionof{table}{Equilibrium partitions and corresponding $\resolution$ ranges.}
\label{tab:gamma_equilibria}
\begin{tabular}{ll}
\toprule
\textbf{Partition} & \textbf{ $\resolution$ Range} \\
\midrule
$\{0,1,2,3\}$                             & $[0.0, 1/3]$ \\
$\{0,1,2\}, \{3\}$                        & $[1/3, 1.0]$ \\
$\{0,3\}, \{1,2\}$                        & $1.0$ \\
$\{0,3\}, \{1\}, \{2\}$                   & $1.0$ \\
$\{0,1\}, \{2\}, \{3\}$                   & $1.0 $ \\
$\{0,2\}, \{1\}, \{3\}$                   & $1.0 $ \\
$\{0\}, \{1,2\}, \{3\}$                   & $1.0$ \\
$\{0\}, \{1\}, \{2\}, \{3\}$              & $1.0$ \\
\bottomrule
\end{tabular}
\end{minipage}
\end{figure}

\begin{table}[h!]
\centering
\caption{Comparison of better- vs best-response dynamics paths for two values of $\resolution$.}
\small
\begin{tabular}{@{}ccccp{6cm}p{6cm}@{}}
\toprule
\textbf{Case} & $\boldsymbol{\resolution}$ & \textbf{Strategy} & \textbf{Step} & \textbf{Move Description} & \textbf{Gain Details} \\
\midrule

1 & 0.1 & Better & 1 & 2: \{0,3\}, \{1\}, \{2\} $\rightarrow$ \{0,2,3\}, \{1\} 
& Gain: 0.80  Link: 0 $\rightarrow$ 1  Non-link: 0 $\rightarrow$ 1 \\
  &     &        & 2 & 1: \{0,2,3\}, \{1\} $\rightarrow$ \{0,1,2,3\} 
& Gain: 1.70  Link: 0 $\rightarrow$ 2  Non-link: 0 $\rightarrow$ 1 \\
  &     & Best   & 1 & 2: \{0,3\}, \{1\}, \{2\} $\rightarrow$ \{0,3\}, \{1,2\} 
& Gain: 0.90  Link: 0 $\rightarrow$ 1  Non-link: 0 $\rightarrow$ 0 \\
  &     &        & 2 & 0: \{0,3\}, \{1,2\} $\rightarrow$ \{0,1,2\}, \{3\} 
& Gain: 0.90  Link: 1 $\rightarrow$ 2  Non-link: 0 $\rightarrow$ 0 \\
  &     &        & 3 & 3: \{0,1,2\}, \{3\} $\rightarrow$ \{0,1,2,3\} 
& Gain: 0.70  Link: 0 $\rightarrow$ 1  Non-link: 0 $\rightarrow$ 2 \\

\midrule

2 & 0.25 & Better & 1 & 2: \{0,3\}, \{1\}, \{2\} $\rightarrow$ \{0,2,3\}, \{1\} 
& Gain: 0.50  Link: 0 $\rightarrow$ 1  Non-link: 0 $\rightarrow$ 1 \\
  &     &        & 2 & 1: \{0,2,3\}, \{1\} $\rightarrow$ \{0,1,2,3\} 
& Gain: 1.25  Link: 0 $\rightarrow$ 2  Non-link: 0 $\rightarrow$ 1 \\
  &     & Best   & 1 & 2: \{0,3\}, \{1\}, \{2\} $\rightarrow$ \{0,3\}, \{1,2\} 
& Gain: 0.75  Link: 0 $\rightarrow$ 1  Non-link: 0 $\rightarrow$ 0 \\
  &     &        & 2 & 0: \{0,3\}, \{1,2\} $\rightarrow$ \{0,1,2\}, \{3\} 
& Gain: 0.75  Link: 1 $\rightarrow$ 2  Non-link: 0 $\rightarrow$ 0 \\
  &     &        & 3 & 3: \{0,1,2\}, \{3\} $\rightarrow$ \{0,1,2,3\} 
& Gain: 0.25  Link: 0 $\rightarrow$ 1  Non-link: 0 $\rightarrow$ 2 \\

\bottomrule
\end{tabular} \label{tab:betterbest}
\end{table}

\begin{figure}[h!]
\centering
\begin{tikzpicture}[->, >=stealth, shorten >=1pt, auto, node distance=3.5cm,
                    semithick, every node/.style={rectangle, draw=black, fill=none, rounded corners}]

\node (a) at (0,0)      {$\{0,3\}, \{1\}, \{2\}$};
\node (b) at (4.5,2)    {$\{0,2,3\}, \{1\}$};
\node (c) at (9,0)      {$\{0,1,2,3\}$};
\node (d) at (2.5,-2)   {$\{0,3\}, \{1,2\}$};
\node (e) at (6.5,-2)   {$\{0,1,2\}, \{3\}$};

\path[->, thick, dark_blue]
    (a) edge[bend left=15] node[above left] {\small Gain = $1-2\resolution$} (b)
    (b) edge[bend left=15] node[above right] {\small Gain = $2-3\resolution$} (c);

\path[->, thick, dark_red, dashed]
    (a) edge[bend right=15] node[below left] {\small Gain = $1-\resolution$} (d)
    (d) edge node[below, yshift=-0.3cm] {\small Gain = $1-\resolution$} (e)
    (e) edge[bend right=10] node[below right] {\small Gain =$1-3\resolution$} (c);

\node[draw=none, fill=none] at (0,-3.3) {
\begin{tabular}{ll}
\textcolor{dark_blue}{\textbf{Better-response dynamics}} & \textcolor{dark_blue}{(solid blue)} \\
\textcolor{dark_red}{\textbf{Best-response dynamics}}    & \textcolor{dark_red}{(dashed red)}
\end{tabular}
};

\end{tikzpicture}

\caption{
Transition graph for $0 < \resolution < {1}/{3}$. Nodes represent partitions, and edges correspond to improving moves labeled with their utility gains. In both paths, the total potential gain is $3 - 5\resolution$. Under better-response dynamics (solid blue), the algorithm selects a smaller initial gain, followed by a larger subsequent gain, leading to faster convergence. In contrast, best-response dynamics (dashed red) selects the move with the highest immediate gain, but this choice results in smaller subsequent gains and more steps before reaching the grand coalition equilibrium.
} \label{fig:examplegraph}
\end{figure}
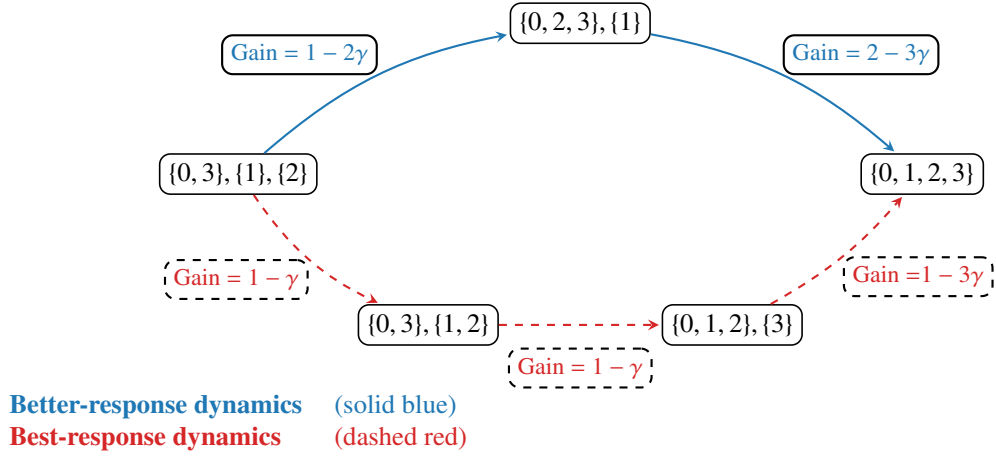

These findings reinforce the central message of our theoretical analysis. As shown in Theorem~\ref{THEO:EQUILIBRIUM}, the convergence time of the better-response dynamics depends on the inverse of the resolution parameter’s denominator (i.e., $\resolution = b/c \Rightarrow$ time $\propto c$), which leads to pseudo-polynomial complexity when $\resolution$ is part of the input. In contrast, Theorem~\ref{theo:extremal} shows that
for extremal values of $\resolution$ (specifically $\resolution \in \{0,1\}$), the convergence time is bounded by $O(\nvertices^2)$, independent of~$\resolution$.
It should be noted that a best-response dynamic is a specific instance of a better-response dynamic. Therefore, the worst-case pseudo-polynomial upper bound on convergence time established in Theorem~\ref{THEO:EQUILIBRIUM} also applies to best-response dynamics. Whether a tighter, more specific upper bound can be proven for the subset of best-response dynamics, which are more restrictive, remains an interesting open problem for future work.

Theorem~\ref{theo:bestbetter} adds another layer to this picture,
by showing that the number of steps to equilibrium also depends on the dynamics chosen (better- vs. best-response dynamics), and that best-response dynamics is not always superior in practice.
Furthermore, as shown in~\cite{traag2019leiden}, better-response dynamics can be crucial for finding certain optimal partitions that are unreachable by strictly greedy best-response dynamics.
These results suggest a nuanced view of convergence, where the structure of the graph, the resolution parameter, and the choice of dynamics all play a significant role. A natural direction for future work is to explore whether lookahead strategies or reinforcement learning can help agents select paths that balance short-term and long-term gains more effectively.

\subsubsection{Summary}
\label{sec:summary_convergence}

Theorem~\ref{THEO:EQUILIBRIUM} establishes that Algorithm~\ref{alg:simple-better-response} converges to an equilibrium partition in pseudo-polynomial time. This result has both theoretical and practical implications. Theoretically, it guarantees convergence under \textit{rational} resolution parameters, addressing a key concern in game-theoretic community detection. Practically, the algorithm's behavior is influenced by three \textit{design choices}:
\begin{enumerate}[i)]
    \item the selection of the resolution parameter $\resolution$, which governs the granularity of the detected communities;
    \item the initial partition of the graph; and
    \item the strategy for determining both the next candidate node to evaluate and the target community to which it may move when more than two communities are considered (lines~\ref{for:node} to~\ref{let:kappa} of Algorithm~\ref{alg:simple-better-response}).
\end{enumerate}

Our work focuses on the first two choices. For (i), Section~\ref{sec:robustness} demonstrates that partitions satisfying equilibrium conditions across a wide range of $\resolution$ values exhibit higher correlation with ground truth. For (ii), Section~\ref{sec:community_tracking} demonstrates that when the initial partition retains partial ground-truth information (such as in community tracking, where a once-accurate partition becomes outdated due to network evolution) the hedonic game can effectively use this prior knowledge to recover a partition much closer to the updated ground truth.

Design choice (iii) involves two interrelated aspects: node selection order, and target community determination during node movements.
Crucially, our convergence proof remains valid for any node selection order, as the minimum potential gain per move (established in Theorem~\ref{THEO:EQUILIBRIUM}) is the same regardless of the node processing sequence.
Regarding community targeting, the conventional greedy approach (always moving nodes to communities offering maximum potential gain) dominates current practice.
However, strategic alternatives permitting temporary regressions (negative potential changes) could theoretically enable escape from local optima.
This concept is central to metaheuristics like Simulated Annealing   that use   non-improving moves to escape local optima.

\subsection{The Robustness of Community Structures}
\label{sec:robustness}

We consider the problem of evaluating an equilibrium by its robustness, which we analyze from two perspectives: the stability of individual nodes and the nature of their potential moves. Section~\ref{sec:resolution_threshold} explores how the resolution parameter, $\resolution$, governs a node's decision-making, particularly for choices that involve a trade-off. In Section~\ref{sec:partial-robust}, we examine partitions that remain stable across a range of resolution values. Finally, Section~\ref{sec:fully-robust} investigates fully-robust partitions, where every node is stable for all possible resolution values. Section~\ref{sec:summary_robustness} then summarizes these findings to offer a unified perspective on partition stability.

\subsubsection{The Familiarity Index as a Decision Threshold for Moves}
\label{sec:resolution_threshold}

\begin{table}[t]
    \caption{Analysis of a potential move from community $\clusterB$ to community $\clusterA$ based on the resolution threshold $\resolution^\star$. For a frustrated choice, the move's incentive depends on the comparison between the global parameter $\resolution$ and the local threshold $\resolution^\star$ (which is precisely the \textbf{Familiarity Index}).}
    \centering \resizebox{\textwidth}{!}{
    \begin{tabular}{|c||c|c|l|}
        \hline
        \textbf{Subcase} & \textbf{Case 1: $\Delta \hat{d} + \Delta d > 0$} & \textbf{Case 2: $\Delta \hat{d} + \Delta d < 0$}  & \textbf{Nature of the Choice} \\
        & ($n_{\clusterA} > n_{\clusterB} - 1)$ & $(n_{\clusterA} < n_{\clusterB} - 1)$ & \\
        \hline \hline
        (a) $\Delta d \geq 0$, $\Delta \hat{d} \leq 0$ & $\resolution^\star \geq 1$ & $\resolution^\star \leq 0$ & The move to $\clusterA$ is unambiguously preferred, regardless of $\resolution$. \\
        \hline
        (b) $\Delta d \leq 0$, $\Delta \hat{d} \geq 0$ & $\resolution^\star \leq 0$  & $\resolution^\star \geq 1$ & The move to $\clusterA$ is unambiguously unpreferred, regardless of $\resolution$. \\
        \hline
        (c) $\Delta d \geq 0$, $\Delta \hat{d} \geq 0$ & $0 < \resolution^\star < 1$ & N/A & This is a frustrated choice. The move yields gain if $\resolution < \resolution^\star$. \\
        \hline
        (d) $\Delta d \leq 0$, $\Delta \hat{d} \leq 0$ & N/A & $0 < \resolution^\star < 1$ & This is a frustrated choice. The move yields gain if $\resolution > \resolution^\star$. \\
        \hline
    \end{tabular} }
    \label{tab:node-stability}
\end{table}

The resolution parameter, $\resolution$, is crucial for resolving the trade-offs inherent in community detection. For a node facing a frustrated choice (a conflict between joining a community with more friends versus one with fewer strangers) $\resolution$ acts as the decisive factor. In this section, we formalize this relationship. We demonstrate that a node's decision to move is determined by comparing $\resolution$ to a critical threshold, which is precisely the \textbf{Familiarity Index}, $F_\anode$, introduced in Eq.~\eqref{eq:familiarity_index}.

For a node $\anode$ currently in community $\clusterB$, a move to a different community $\clusterA$ is beneficial only if it increases its potential. The node is therefore stable if, for all alternative communities $\clusterA \neq \clusterB$, the following condition holds:
\begin{equation}
    \potentialNode_{\anode}^\resolution(\clusterB) \geq \potentialNode_{\anode}^\resolution(\clusterA \cup \{\anode\}), \quad \forall \clusterA \neq \clusterB.
\end{equation}
By substituting the definition of node potential from Eq.~\eqref{eq:potential_node1} and rearranging the terms, we can express the stability condition as:
\begin{equation}
(1 - \resolution)d_{\anode}^\clusterB - \resolution(n_\clusterB - d_{\anode}^\clusterB - 1) \geq (1 - \resolution)d_{\anode}^\clusterA - \resolution(n_\clusterA - d_{\anode}^\clusterA) \Rightarrow  d_{\anode}^\clusterA - d_{\anode}^\clusterB \leq \resolution (n_\clusterA - n_\clusterB + 1).
\end{equation}
Let the difference in neighbors, non-neighbors, and number of nodes in the communities be given by
\begin{equation}
  \Delta d =   d_{\anode}^\clusterA - d_{\anode}^\clusterB, \quad \Delta \hat{d} = n_\clusterA - d_{\anode}^\clusterA - ( n_\clusterB - d_{\anode}^\clusterB - 1), \quad \Delta n = n_\clusterA - n_\clusterB +1. 
\end{equation}
As shown in the introduction, this total change in community composition can be decomposed into the change in friends and the change in strangers, $\Delta n = \Delta d + \Delta \hat{d}$. The stability condition thus becomes:
\begin{equation}
\Delta d \leq \resolution (\Delta d + \Delta \hat{d}) \Rightarrow \frac{\Delta d }{ \Delta \hat{d} + \Delta d} \leq \resolution.
\end{equation}
Let  $\resolution^\star$ be the critical resolution value for which the potential gain corresponding to a move of node $i$ from 
$\clusterB$ to $\clusterA$ is zero. This threshold is given by:
%
\begin{equation}
    \resolution^{\star} = \frac{d_{\anode}^\clusterA - d_{\anode}^\clusterB}{n_\clusterA - n_\clusterB + 1} = \frac{d_{\anode}^\clusterA - d_{\anode}^\clusterB}{ d_{\anode}^\clusterA - d_{\anode}^\clusterB  + n_\clusterA - d_{\anode}^\clusterA - ( n_\clusterB - d_{\anode}^\clusterB - 1) } = \frac{\Delta d}{\Delta d + \Delta \hat{d}}=\frac{\Delta d}{\Delta n}.
    \label{eq:resolution_lower_bound}
\end{equation}
Comparing this result with Eq.~\eqref{eq:familiarity_index}, we see that this critical threshold is identical to the Familiarity Index for the potential move: $\resolution^\star = F_\anode(\clusterB, \clusterA)$.
To simplify presentation, we assume that $\Delta \hat{d} + \Delta d \neq 0$, and that either $\Delta \hat{d} \neq 0$ or $\Delta {d} \neq 0$. In addition, we assume $\resolution \neq \resolution^{\star}$.

Consequently, the stability of a node hinges on a direct comparison between the global   resolution parameter $\resolution$ and the locally computed Familiarity Index $\resolution^\star$. The outcome depends on the nature of the move, specifically whether the destination community is larger or smaller in size ($\Delta n > 0$ or $\Delta n < 0$):
\begin{itemize}
    \item \textbf{Case 1: Moving to a larger effective community ($\Delta n > 0$)}. A move is beneficial if $\resolution < \resolution^\star$. Here, a low $\resolution$ prioritizes friend-maximization. The node remains stable if $\resolution \geq \resolution^\star$, where the penalty for gaining strangers outweighs the benefit of new friends.
    \item \textbf{Case 2: Moving to a smaller effective community ($\Delta n < 0$)}. The inequality flips, and a move becomes beneficial if $\resolution > \resolution^\star$. Here, a high $\resolution$ rewards the reduction in strangers, making the move attractive. The node remains stable if $\resolution \leq \resolution^\star$.
\end{itemize}

Table~\ref{tab:node-stability} provides a comprehensive summary of these conditions.
The non-frustrated subcases (a) and (b) correspond to unambiguous choices where the outcome is independent of $\resolution \in [0,1]$, as the Familiarity Index $\resolution^\star$ falls outside the $(0,1)$ range. The frustrated subcases (c) and (d) are where $\resolution^\star \in (0,1)$, and the node's stability becomes critically dependent on the choice of $\resolution$. This analysis provides a clear bridge between the micro-level incentives of individual nodes and the macro-level behavior of the CPM, grounding it in the interpretable trade-off captured by the Familiarity Index.

\subsubsection{$[\resolution_0, \resolution_1]\--$Robust Equilibrium}
\label{sec:partial-robust}


What is the impact of $\resolution$ on  an equilibrium? There are  partitions that correspond to an equilibrium for a given value of $\resolution$ and that remain an equilibrium after $\resolution$   varies within a range of values. We refer to an equilibrium as $[\resolution_0,\resolution_1]\--$\emph{robust} if it remains an equilibrium $\forall {\resolution} \in  [\resolution_0,\resolution_1]$. We denote by  $\partition^{[\resolution_0, \resolution_1]}$ a partition that is an equilibrium for $\resolution \in [\resolution_0, \resolution_1]$. We hypothesize that if an equilibrium is robust over a broader range, it is more likely to accurately capture the ground truth.

\begin{figure}[t]
    \centering
    \includegraphics[width=0.7\textwidth]{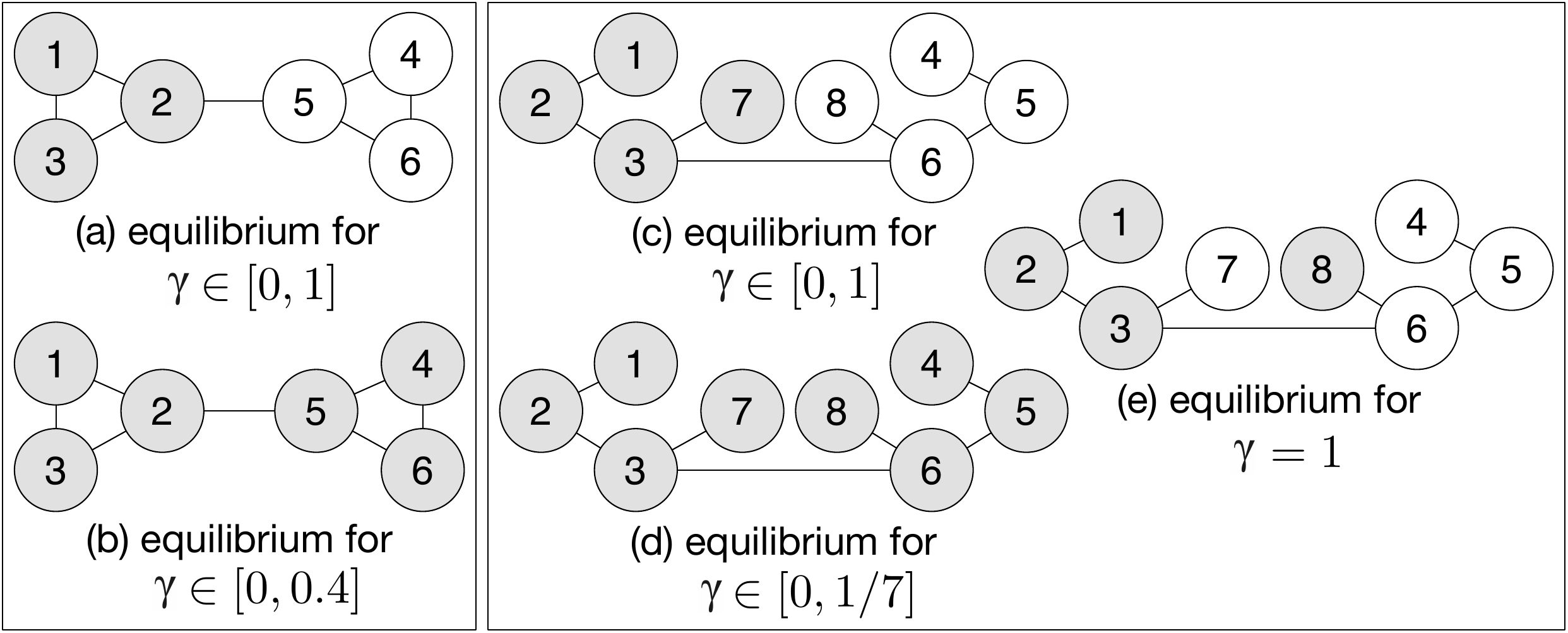}
    \vspace{-0.1in}
    \caption{Robustness of hedonic equilibria.} 
    \label{fig:toy}
\end{figure}

Next, we illustrate  through two examples the intuitive idea that  if an equilibrium is robust over a broader
range of values, it is more likely to accurately capture the ground truth.
Figure~\ref{fig:toy}(a) shows a 6-node network comprising two cliques of 3 nodes connected through a single edge. Communities $\clusterA$ and $\clusterB$ correspond to nodes $\{1,2,3\}$ and $\{4,5,6\}$, respectively. If $\resolution=0$, the hedonic value of each node equals $2$, which is the number of neighbors of each node in its community. If one of the nodes $1, 3, 4$ or $6$ deviates to the other community, its hedonic value decreases to $0$. If nodes $2$ or $5$ deviate, their hedonic value decreases to $1$. For $\resolution=0$, the configuration wherein all nodes are in the same community (grand coalition) is also an equilibrium (Fig.~\ref{fig:toy}(b)). However, the latter equilibrium is less robust than the former, as the grand coalition   holds as an equilibrium   for $\resolution \in [0,0.4]$
whereas the two-community symmetric equilibrium holds for $\resolution \in [0, 1]$.
It is important to note that these partitions are \textit{equilibria}, i.e., no single node has an incentive to unilaterally move, but are not necessarily the global optima of the partition potential function for the entire range of $\resolution$.

In another example (Figures~\ref{fig:toy}(c)-(e)), the equilibrium that minimizes the number of edges across communities, different from the grand coalition, is the most robust (Fig.~\ref{fig:toy}(c)). It holds for   $\resolution \in [0, 1]$, whereas the other equilibria hold for a strict subset of values of $\resolution$, as indicated in Figure~\ref{fig:toy}. In particular, the grand coalition holds for    $\resolution \in [0, 1/7]$ (Fig.~\ref{fig:toy}(d)).

The following theorem indicates that to show a given equilibrium holds for a range of values of $\resolution$, it suffices to consider the two extremes of the range.

\begin{theorem} \label{THEO:RANGE_EQ}
    Given a $\resolution$-equilibrium partition $\partition^\resolution$ that holds as equilibrium for $\resolution=\resolution_0$ and for $\resolution=\resolution_1$, with $\resolution_1 > \resolution_0$,  then $\partition^\resolution$ is  an $[\resolution_0, \resolution_1]\--$robust equilibrium, denoted as  $\partition^{[\resolution_0,  \resolution_1]}$.
\end{theorem}

\begin{proof}

We consider a given node $\anode$ of interest, part of community $\clusterB$. Then, we consider a tagged community $\clusterA$, with $\clusterA \neq \clusterB$.
We show that node $\anode$ has no incentive to change from $\clusterB$ to   $\clusterA$, for any $\resolution \in [\resolution_0, \resolution_1]$, as long as it has no incentive to change for $\resolution=\resolution_0$ and $\resolution=\resolution_1$. Using the same terminology as in Theorem~\ref{THEO:EQUILIBRIUM} and leveraging~\eqref{eq:delta_node}, $\Delta \potentialNode_{\anode}^{\clusterB\clusterA}(\resolution)$ is a linear function of $\resolution$, so $\Delta \potentialNode_{\anode}^{\clusterB\clusterA}(\resolution_0) \leq 0$ and $\Delta \potentialNode_{\anode}^{\clusterB\clusterA}(\resolution_1) \leq 0$ imply $\Delta \potentialNode_{\anode}^{\clusterB\clusterA}(\resolution^*) \leq 0$ for $\resolution^* \in [\resolution_0, \resolution_1]$. This holds for any node in community $\clusterB$. By symmetry, the same arguments apply for nodes in community $\clusterA$. As communities  $\clusterA$ and $\clusterB$ were chosen arbitrarily, the argument holds for any such pair of communities. Thus, the proof is complete.
\end{proof}

\subsubsection{Fully-Robust Partitions}
\label{sec:fully-robust}

Recall that a partition $\partition^{[0,1]}$  is a partition that holds as an equilibrium for all $\resolution$, where $\resolution \in [0,1].$ We denote such partition as a fully-robust partition. Next, we further discuss the characterization of fully-robust partitions, and provide a simple criterion to find balanced fully-robust partitions.

\paragraph{\textbf{Characterization}}
Recall from Eq.~\eqref{eq:robustness} that a node $\anode$ is robust if its current community, $\sigma_\anode$, is the optimal choice compared to all other communities, meaning it simultaneously maximizes the number of neighbors ($d^{\sigma_\anode}_{\anode}$) and minimizes the number of non-neighbors ($\hat{d}^{\sigma_\anode}_{\anode}$). The indicator function $R(i, \partition)$ captures this condition, evaluating to 1 if the node is robust and 0 otherwise.

\begin{definition}[Fully-robust partition]\label{def:fully-robust}
A partition is fully-robust if  $R(i,\pi)=1$ for  every node $i$.  
\end{definition}
%
%
The equivalence between this definition and $[0,1]$-robustness follows directly from Theorem~\ref{THEO:RANGE_EQ} and the linearity of potential functions. Indeed, if a partition is robust for $\resolution=0$, it means that $d^{\nodecomm_{\anode}}_{\anode} \geq d^{\community_k}_{\anode}$ for all $k$,    for all nodes.  Similarly, if a partition is robust for $\resolution=1$, it means that $\hat{d}^{\nodecomm_{\anode}}_{\anode} \leq \hat{d}^{\community_k}_{\anode}$, for all $k$,    for all nodes.  Therefore,  a partition is $[0,1]$-robust if and only if it is fully-robust according to the above definition.

\paragraph{\textbf{Relation to Frustrated Potts Model}}
The frustrated Potts model encompasses a value function given by Eq.~\eqref{eq:pair_value}. \emph{Frustration} refers to the competition between ferromagnetic 
($A_{ij} = 1$, and $ 1-\resolution >0$) and antiferromagnetic ($A_{ij} = 0$, and $ -\resolution <0$) interactions in a heterogeneous system.  
In statistical physics,  some systems with competing interactions
cannot achieve a configuration that satisfies all local constraints simultaneously.  
Similarly, in the model considered in this work, a node faces a frustrated choice if a potential move requires trading off the number of neighbors against the number of non-neighbors.
A non-robust node might face either frustrated or unambiguous choices when considering a move. If all available moves that offer any improvement are frustrated, the node must make a trade-off.
Table~\ref{tab:node-stability} summarizes, in subcases (c) and (d), the conditions corresponding to  frustrated choices. A  non-robust node can exist in an equilibrium if all its potential (frustrated) moves are disincentivized by the current value of $\resolution$.

\paragraph{\textbf{Illustrative Examples}}
Figure~\ref{fig:toy} illustrates two examples of   fully-robust partitions that better capture communities  than alternative equilibria.
The concept of a fully robust partition is flexible, as its stability is determined by the set of available communities a node can move to. Figure~\ref{fig:toy-coalitions} illustrates this by showing multiple partitions that can be considered fully robust, depending on the constraints imposed on community formation. The most critical constraint is whether a node is permitted to \emph{isolate itself} into a new, singleton community, irrespectively of the current partition state.

\paragraph{\textbf{Node Isolation}}
If node isolation is allowed, any node in a community with at least one non-neighbor (a ``stranger'') will prefer to form its own community when the resolution parameter $\resolution=1$, since an isolated node has zero strangers. In such an unconstrained scenario, only partitions where every community is a clique could be fully robust. However, by limiting the available choices, for example, by setting a maximum number of communities, $K$, the notion of full robustness becomes more applicable  (see Assumption~\ref{ass:mainassumption}).
This constraint is what allows non-trivial partitions, like the macro-level divisions in Figure~\ref{fig:toy-coalitions}(b), to be stable. A node in this partition is fully robust because, given the fixed set of communities, it has no better option. Even the grand coalition (Figure~\ref{fig:toy-coalitions}(a)) can be trivially considered fully robust if $K=1$, as no alternative communities exist. Therefore, full robustness is not an absolute property but rather a state of equilibrium relative to a defined set of possible partitions.

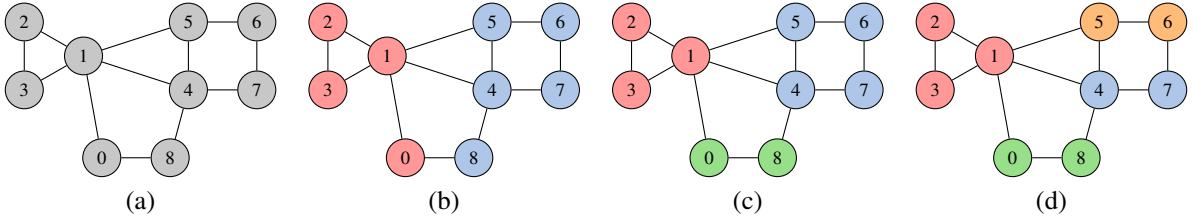
\begin{figure}[t!]
    \centering
    \begin{tabular}{cccc}
    \begin{tikzpicture}[scale=0.6, every node/.style={mynode}] 
        \pgfmathsetmacro{\ud}{1.0} 
        \pgfmathsetmacro{\dsq}{1.5*\ud} 
        \pgfmathsetmacro{\gap}{0.4*\ud} 

        \pgfmathsetmacro{\ylow}{-0.5*\dsq}
        \pgfmathsetmacro{\yhigh}{\ylow + \dsq}
        \pgfmathsetmacro{\ymid}{\ylow + 0.5*\dsq}
        \pgfmathsetmacro{\ybottom}{-1.5*\dsq}

        \pgfmathsetmacro{\xeight}{0}
        \pgfmathsetmacro{\xzero}{-\dsq}
        
        \pgfmathsetmacro{\xfour}{\xeight + \gap}
        \pgfmathsetmacro{\xfive}{\xfour}
        \pgfmathsetmacro{\xsix}{\xfour + \dsq}
        \pgfmathsetmacro{\xseven}{\xsix}

        \pgfmathsetmacro{\xone}{\xzero - \gap}
        \pgfmathsetmacro{\xtri}{\xone - sqrt(0.75)*\dsq}
        \pgfmathsetmacro{\xtwo}{\xtri}
        \pgfmathsetmacro{\xthree}{\xtri}

        \node (n0) at (\xzero, \ybottom) {0};
        \node (n8) at (\xeight, \ybottom) {8};
        \node (n1) at (\xone, \ymid) {1};
        \node (n2) at (\xtwo, \yhigh) {2};
        \node (n3) at (\xthree, \ylow) {3};
        \node (n4) at (\xfour, \ylow) {4};
        \node (n5) at (\xfive, \yhigh) {5};
        \node (n6) at (\xsix, \yhigh) {6};
        \node (n7) at (\xseven, \ylow) {7};

        \draw (n0) -- (n1); \draw (n0) -- (n8); \draw (n1) -- (n2); \draw (n1) -- (n3);
        \draw (n1) -- (n4); \draw (n1) -- (n5); \draw (n2) -- (n3); \draw (n4) -- (n5);
        \draw (n5) -- (n6); \draw (n6) -- (n7); \draw (n7) -- (n4); \draw (n4) -- (n8);
    \end{tikzpicture} &
    \begin{tikzpicture}[scale=0.6, every node/.style={mynode}]
        \pgfmathsetmacro{\ud}{1.0}
        \pgfmathsetmacro{\dsq}{1.5*\ud}
        \pgfmathsetmacro{\gap}{0.4*\ud}

        \pgfmathsetmacro{\ylow}{-0.5*\dsq}
        \pgfmathsetmacro{\yhigh}{\ylow + \dsq}
        \pgfmathsetmacro{\ymid}{\ylow + 0.5*\dsq}
        \pgfmathsetmacro{\ybottom}{-1.5*\dsq}

        \pgfmathsetmacro{\xeight}{0}
        \pgfmathsetmacro{\xzero}{-\dsq}
        
        \pgfmathsetmacro{\xfour}{\xeight + \gap}
        \pgfmathsetmacro{\xfive}{\xfour}
        \pgfmathsetmacro{\xsix}{\xfour + \dsq}
        \pgfmathsetmacro{\xseven}{\xsix}

        \pgfmathsetmacro{\xone}{\xzero - \gap}
        \pgfmathsetmacro{\xtri}{\xone - sqrt(0.75)*\dsq}
        \pgfmathsetmacro{\xtwo}{\xtri}
        \pgfmathsetmacro{\xthree}{\xtri}

        \node [rednode] (n0) at (\xzero, \ybottom) {0};
        \node [bluenode] (n8) at (\xeight, \ybottom) {8};
        \node [rednode] (n1) at (\xone, \ymid) {1};
        \node [rednode] (n2) at (\xtwo, \yhigh) {2};
        \node [rednode] (n3) at (\xthree, \ylow) {3};
        \node [bluenode] (n4) at (\xfour, \ylow) {4};
        \node [bluenode] (n5) at (\xfive, \yhigh) {5};
        \node [bluenode] (n6) at (\xsix, \yhigh) {6};
        \node [bluenode] (n7) at (\xseven, \ylow) {7};

        \draw (n0) -- (n1); \draw (n0) -- (n8); \draw (n1) -- (n2); \draw (n1) -- (n3);
        \draw (n1) -- (n4); \draw (n1) -- (n5); \draw (n2) -- (n3); \draw (n4) -- (n5);
        \draw (n5) -- (n6); \draw (n6) -- (n7); \draw (n7) -- (n4); \draw (n4) -- (n8);
    \end{tikzpicture} &
    \begin{tikzpicture}[scale=0.6, every node/.style={mynode}]
        \pgfmathsetmacro{\ud}{1.0}
        \pgfmathsetmacro{\dsq}{1.5*\ud}
        \pgfmathsetmacro{\gap}{0.4*\ud}

        \pgfmathsetmacro{\ylow}{-0.5*\dsq}
        \pgfmathsetmacro{\yhigh}{\ylow + \dsq}
        \pgfmathsetmacro{\ymid}{\ylow + 0.5*\dsq}
        \pgfmathsetmacro{\ybottom}{-1.5*\dsq}

        \pgfmathsetmacro{\xeight}{0}
        \pgfmathsetmacro{\xzero}{-\dsq}
        
        \pgfmathsetmacro{\xfour}{\xeight + \gap}
        \pgfmathsetmacro{\xfive}{\xfour}
        \pgfmathsetmacro{\xsix}{\xfour + \dsq}
        \pgfmathsetmacro{\xseven}{\xsix}

        \pgfmathsetmacro{\xone}{\xzero - \gap}
        \pgfmathsetmacro{\xtri}{\xone - sqrt(0.75)*\dsq}
        \pgfmathsetmacro{\xtwo}{\xtri}
        \pgfmathsetmacro{\xthree}{\xtri}

        \node [greennode] (n0) at (\xzero, \ybottom) {0};
        \node [greennode] (n8) at (\xeight, \ybottom) {8};
        \node [rednode] (n1) at (\xone, \ymid) {1};
        \node [rednode] (n2) at (\xtwo, \yhigh) {2};
        \node [rednode] (n3) at (\xthree, \ylow) {3};
        \node [bluenode] (n4) at (\xfour, \ylow) {4};
        \node [bluenode] (n5) at (\xfive, \yhigh) {5};
        \node [bluenode] (n6) at (\xsix, \yhigh) {6};
        \node [bluenode] (n7) at (\xseven, \ylow) {7};

        \draw (n0) -- (n1); \draw (n0) -- (n8); \draw (n1) -- (n2); \draw (n1) -- (n3);
        \draw (n1) -- (n4); \draw (n1) -- (n5); \draw (n2) -- (n3); \draw (n4) -- (n5);
        \draw (n5) -- (n6); \draw (n6) -- (n7); \draw (n7) -- (n4); \draw (n4) -- (n8);
    \end{tikzpicture} &
    \begin{tikzpicture}[scale=0.6, every node/.style={mynode}]
        \pgfmathsetmacro{\ud}{1.0}
        \pgfmathsetmacro{\dsq}{1.5*\ud}
        \pgfmathsetmacro{\gap}{0.4*\ud}

        \pgfmathsetmacro{\ylow}{-0.5*\dsq}
        \pgfmathsetmacro{\yhigh}{\ylow + \dsq}
        \pgfmathsetmacro{\ymid}{\ylow + 0.5*\dsq}
        \pgfmathsetmacro{\ybottom}{-1.5*\dsq}

        \pgfmathsetmacro{\xeight}{0}
        \pgfmathsetmacro{\xzero}{-\dsq}
        
        \pgfmathsetmacro{\xfour}{\xeight + \gap}
        \pgfmathsetmacro{\xfive}{\xfour}
        \pgfmathsetmacro{\xsix}{\xfour + \dsq}
        \pgfmathsetmacro{\xseven}{\xsix}

        \pgfmathsetmacro{\xone}{\xzero - \gap}
        \pgfmathsetmacro{\xtri}{\xone - sqrt(0.75)*\dsq}
        \pgfmathsetmacro{\xtwo}{\xtri}
        \pgfmathsetmacro{\xthree}{\xtri}

        \node [greennode] (n0) at (\xzero, \ybottom) {0};
        \node [greennode] (n8) at (\xeight, \ybottom) {8};
        \node [rednode] (n1) at (\xone, \ymid) {1};
        \node [rednode] (n2) at (\xtwo, \yhigh) {2};
        \node [rednode] (n3) at (\xthree, \ylow) {3};
        \node [bluenode] (n4) at (\xfour, \ylow) {4};
        \node [orangenode] (n5) at (\xfive, \yhigh) {5};
        \node [orangenode] (n6) at (\xsix, \yhigh) {6};
        \node [bluenode] (n7) at (\xseven, \ylow) {7};

        \draw (n0) -- (n1); \draw (n0) -- (n8); \draw (n1) -- (n2); \draw (n1) -- (n3);
        \draw (n1) -- (n4); \draw (n1) -- (n5); \draw (n2) -- (n3); \draw (n4) -- (n5);
        \draw (n5) -- (n6); \draw (n6) -- (n7); \draw (n7) -- (n4); \draw (n4) -- (n8);
    \end{tikzpicture} \\
    (a) & (b) & (c) & (d) \\
    \end{tabular}

    \caption{From a macro view of two large coalitions to a more granular micro-level segmentation, each figure shows a fully-robust partition according to Definition~\ref{def:fully-robust}: (a) original (unlabeled) network; (b) two-communities; (c) third communities emerges; (d) micro-level communities.
    Note that these partitions are considered fully-robust only under the constraint that nodes cannot unilaterally isolate themselves. If a node could form a new community, it would prefer to become a singleton to eliminate all ``strangers," which is the optimal state when the resolution $\resolution=1$. This restriction on isolation is what allows non-clique-based communities to be stable across all $\resolution \in [0,1]$ (see Assumption~\ref{ass:mainassumption}).
    }
    \label{fig:toy-coalitions}
\end{figure}

\paragraph{\textbf{Criterion to Find Fully-Robust Partitions}}
Next, we provide a simple criterion to find fully-robust partitions.

\begin{theorem}
    \label{THEO:ROBUST_EQ}
    An equilibrium partition $\partition^{\resolution}$, which 1) holds for $\resolution=0$,  and  2)  such that  all  communities have the same number of nodes, is an equilibrium for any $\resolution \in [0,1]$.
\end{theorem}

The \emph{imbalance} of a partition is the difference between the sizes of its largest and smallest communities. In a \emph{balanced} partition, communities have the same size. The proof of the above theorem involves showing that  balanced equilibria for $\resolution=0$, e.g., Figs.~\ref{fig:toy}(a) and~\ref{fig:toy}(c), are also equilibria for $\resolution=1$. Then, the result follows from Theorem~\ref{THEO:RANGE_EQ}.

\begin{proof}

We consider a given tagged node $\anode$, part of community $\clusterB$ in the equilibrium partition $\partition^{\resolution}$, for $\resolution=0$.  Using the same terminology as in Theorem~\ref{THEO:RANGE_EQ} and leveraging Eq.~\eqref{eq:delta_node}, the proof proceeds as follows. 

We consider a tagged community $\clusterA$, with $\clusterA \neq \clusterB$.
We show that node $\anode$ has no incentive to change from $\clusterB$ to   $\clusterA$, for  $\resolution=1$.  
Let $\potentialNode_{\anode}^{\resolution}$ and $\tilde{\potentialNode}_{\anode}^{\resolution}$ be the hedonic potential value of the tagged node before and after a simulated change from community $\clusterB$ to $\clusterA$.
It follows from Eqs.~\eqref{eq:potential_node} and~\eqref{eq:delta_node} that

\begin{itemize}
    \item for $\resolution=0$, $\partition^\resolution$ is an equilibrium if and only if 
 $\potentialNode_{\anode}^0 \ge \tilde{\potentialNode}_{\anode}^0$, i.e., $\degriB \ge \degriA$
    \item for $\resolution=1$, $\partition^\resolution$ is an equilibrium if and only if  $\potentialNode_{\anode}^1 \ge \tilde{\potentialNode}_{\anode}^1$, i.e., $\cdegriB \le \cdegriA$
\end{itemize}

When assuming communities of the same size, $n_\clusterA = n_\clusterB = {\nvertices}/{K}$, if the first condition $\degriB \ge \degriA$ for $\resolution = 0$ holds true, then the second condition $\cdegriB \le \cdegriA$ for $\resolution = 1$ will also hold because $\cdegriB = n_\clusterB - (\degriB + 1)$ and $\cdegriA = n_\clusterA - \degriA$. Therefore, if two communities have the same number of nodes and are part of  an equilibrium for $\resolution=0$, they are also part of  an equilibrium for $\resolution=1$.

As communities  $\clusterA$ and $\clusterB$ were chosen arbitrarily, the argument holds for any such pair of communities. Thus, the proof is complete.
\end{proof}


Theorem~\ref{THEO:ROBUST_EQ} highlights an asymmetry in the impact of the resolution parameter: equilibrium conditions for \(\resolution = 0\) can imply robustness across the entire parameter range, whereas analyzing the robustness of a partition known to be an equilibrium for \(\resolution = 1\) appears more intricate. As illustrated in Figure~\ref{fig:toy}(e), the  analysis  of $\resolution=1$ must account for singleton communities and dangling nodes—such as node 7 in Figure~\ref{fig:toy}(e)—which, despite being connected in the original graph, have a degree of zero within their assigned communities. 
A deeper investigation into this matter is left for future work.


\subsubsection{Summary}
\label{sec:summary_robustness}

In this section, we have analyzed the robustness of community structures by examining their stability across different values of the resolution parameter, $\resolution$. Our analysis bridges the gap between the micro-level incentives of individual nodes and the macro-level stability of the entire partition.

First, we formalized the decision-making process for individual nodes. In Section~\ref{sec:resolution_threshold}, we demonstrated that for any potential move, there exists a critical resolution threshold, $\resolution^\star$, which is equivalent to the \emph{Familiarity Index} (Eq.~\eqref{eq:familiarity_index}).
A node is \textbf{robust} if its preference for its current community is unambiguous and optimal, meaning it is stable for any value of $\resolution$. If a non-robust node can only improve its utility by making a trade-off, that potential move is a  frustrated choice, and the decision hinges on whether the global $\resolution$ is above or below the local threshold $\resolution^\star$.

Next, we extended this concept from individual nodes to entire partitions. In Section~\ref{sec:partial-robust}, we introduced the idea of a {$[\resolution_0, \resolution_1]\--$robust equilibrium}—a partition that remains stable for all $\resolution$ within a given interval. Theorem~\ref{THEO:RANGE_EQ} provides a powerful simplification, proving that to verify robustness across an interval, one only needs to check for stability at its endpoints, $\resolution_0$ and $\resolution_1$. This result stems from the linear relationship between a node's potential gain and the resolution parameter.

Finally, in Section~\ref{sec:fully-robust}, we examined the strongest form of stability: {fully-robust partitions}, which are equilibria for all $\resolution \in [0,1]$. We established that such partitions are precisely those in which every node is robust in the strictest sense—simultaneously maximizing its internal neighbors and minimizing its internal non-neighbors. This condition connects our game-theoretic framework to the physical concept of an unfrustrated system. To aid in the identification of these highly stable structures, Theorem~\ref{THEO:ROBUST_EQ} offers a practical criterion: any balanced partition (i.e., one with equal-sized communities) that is an equilibrium for $\resolution = 0$ is guaranteed to be fully-robust.

\section{Empirical Analysis}
\label{sec:empirical_analysis}

This section presents our empirical analysis. In Section~\ref{sec:setup}, we describe the experimental setup, including the synthetic networks, evaluation metrics, and community detection methods used. In Section~\ref{sec:gt_robustness}, we analyze the inherent robustness of ground-truth partitions. Finally, in Section~\ref{sec:community_tracking}, we evaluate the performance of different algorithms in a community tracking scenario.

\subsection{Experimental Setup}
\label{sec:setup}



Our experimental design is detailed in this section. Section~\ref{sec:synthetic_networks} explains how we generate synthetic networks with known ground truth. Section~\ref{sec:metrics} defines the metrics used for evaluation and  Section~\ref{sec:methods} outlines the community detection algorithms  considered in our community tracking experiment. 


\subsubsection{Setup}
\label{sec:synthetic_networks}

\paragraph{\textbf{Probabilistic model}}
We consider a special instance  of the  Stochastic Block Model (SBM)~\cite{HOLLAND1983SBM}, namely the   Planted Partition Model (PPM)~\cite{abbe2015exact}, where nodes within the same community connect with probability $p$ and those in different communities with probability $q$ (with $p>q$ in the assortative case). Motivated by  Theorem~\ref{THEO:ROBUST_EQ},  we further focus on equal-sized communities, which correspond to the Symmetric Assortative PPM (SAPPM).  SAPPM produces graphs with $K$ equal-size communities of $N$ nodes each,  leaving us with four parameters, $(K,N,p,q)$. 
 

\paragraph{\textbf{Difficulty parameter}}
To quantify the challenge of community detection, we define the difficulty parameter \(\difficulty\) as \(\difficulty = q/p\). This parameter captures the relative strength of inter-community connections: higher values of \(\difficulty\) indicate greater difficulty in distinguishing communities. In the extreme case where \(p = q\), community structure becomes indistinguishable, making detection impossible~\cite{abbe2015exact}.


\paragraph{\textbf{Parameters}}
%
%
%
%
Except otherwise noted,   we let $p$ vary from 0.01 to 0.10 in increments of 0.01, and let ${\difficulty}  \in  \{ 0.1, 0.2, 0.3, \\ 0.4, 0.5, 0.55, 0.6, 0.65, 0.7, 0.75\}$. 
In addition, we account for two up to six communities, $2 \leq K \leq 6$.
We empirically observed that those values are adequate to highlight our key insights.


\paragraph{\textbf{Samples for confidence intervals}}
We generated multiple network samples for each parameter set to ensure statistical significance.%
\footnote{Synthetic networks were generated using the Stochastic Block Model (SBM) function from the \texttt{NetworkX} library, chosen specifically for its \texttt{seed} parameter which guarantees the reproducibility of our experiments. All generated graphs were immediately converted to \texttt{igraph} objects for analysis, maintaining consistency with the rest of our methodology. \url{https://networkx.org/documentation/stable/reference/generated/networkx.generators.community.stochastic_block_model}.}
In particular, for each pair of parameters $(p,\difficulty)$ we generated 100 network samples.
This  sufficed for the purpose of producing  95\% confidence intervals of length less than 0.01. 

\subsubsection{Evaluation Metrics}
\label{sec:metrics}

We assess performance across three dimensions:  
\emph{(i) \textbf{Efficiency}}: Measured as execution time in seconds. 
\emph{(ii) \textbf{Robustness}}: Defined per Eq.~\eqref{eq:robustness} as the fraction of robust nodes.   As discussed in  Section~\ref{sec:robustness}, robustness  can be equivalently interpreted as  the  fraction of nodes that satisfy strict max–min criteria or as the fraction of nodes that are stable across all $\resolution \in [0,1]$.
\emph{(iii) \textbf{Accuracy}}: Characterizes the agreement between the partitions produced by the different algorithms and the ground truth. Given a partition $\partition$, we compare it against the ground-truth partition $\partition^{\star}$ using the Adjusted Rand Index (ARI)~\cite{hubert1985comparing}. The ARI measures the similarity between two partitions, correcting for the agreement that would be expected by chance. It ranges from 1 to -1. A score of 1 indicates perfect agreement, while a score near 0 suggests that the similarity is no better than random. Negative scores indicate less agreement than expected by chance.%
\footnote{We calculated the Adjusted Rand Index using the \texttt{igraph.compare\_communities} function setting the parameter \texttt{method=`adjusted\_rand'}.
\url{https://python.igraph.org/en/stable/api/index.html\#compare_communities}.}


Let   $n_k$ be the number of nodes in community $\community_k$ under partition $\pi$ and let $n^\star_{\ell}$ be the number of nodes in community $\community^\star_{\ell}$ under the reference partition $\pi^\star$.
Let $K$ and $K^\star$ be the number of communities in $\pi$ and $\pi^\star$, respectively. 
The expected number of node pairs that would be placed in the same community in both partitions \(\pi\) and \(\pi^\star\) if the assignments were random, but preserving the number and sizes of communities in each partition, is given by:
\[
\mathbb{E}(\graph, \partition, \partition^{\star}) = \left( {\sum_{k=1}^K \binom{n_k}{2} \sum_{\ell=1}^{K^{\star}} \binom{n_{\ell}^{\star}}{2} }\right) \Bigg/ { \binom{\nvertices}{2}}.
\]
 The above quantity serves as a baseline for chance agreement. The ARI metric is given by:
\begin{equation}
\accuracy(\graph, \partition, \partition^{\star}) =
\left(
\sum_{k=1}^{K} \sum_{\ell=1}^{K^\star} \binom{n_{k\ell}}{2} - \mathbb{E}(\graph, \partition, \partition^{\star})
\right) \Bigg/  \left(
\frac{1}{2} \left(\sum_{k=1}^K  \binom{n_k}{2} + \sum_{\ell=1}^{K^\star} \binom{n_{\ell}^{\star}}{2}\right) - \mathbb{E}(\graph, \partition, \partition^{\star})
\right) 
\label{eq:acc}
\end{equation}
where $n_{kl} = |\community_k \cap \community_{\ell}^\star|$. This metric compares the number of node pairs co-clustered in both partitions against the number expected by random chance.

For each of the above three metrics, we report an average over multiple runs per $(p, \difficulty)$ pair.

\subsubsection{Community Detection Methods}
\label{sec:methods}


We compare five community detection methods (for details see~\ref{sec:methodsappendix}):
\begin{itemize}

\item \textbf{Leiden (Full-Fledged)}: the standard Leiden algorithm~\cite{traag2019leiden} (see~\ref{sec:leiden}) optimizing CPM~\cite{traag2011cpm}.%
\footnote{From an implementation standpoint, this involved a minor bug fix to the reference implementation of the Leiden algorithm. \url{https://github.com/igraph/igraph/pull/2799}.}
We disable isolated singletons so that we never exceed $K$ communities (see Assumption~\ref{ass:mainassumption}). The resolution parameter $\resolution$ is set to the   edge density~\cite{avrachenkov2017cooperative}.

\item \textbf{Leiden (Phase~1)}: we apply only the local-move phase of Leiden,
indicating that it suffices for community tracking, specially for small values of $\noise$. As shown in Section~\ref{sec:find_equilibrium}, such method finds a hedonic equilibrium in pseudo-polynomial time.%
\footnote{The C implementation of our extension and its corresponding Python interface can be found at \url{https://github.com/lucaslopes/igraph/tree/physica-a}  and \url{https://github.com/lucaslopes/python-igraph/tree/physica-a}.}

\item \textbf{Spectral Clustering}: a classic eigenvector-based approach (see~\ref{sec:spectral}). It does \emph{not} accept any initial partition, so its outcome is invariant to $\noise$.

\item \textbf{One Pass}: a simple baseline. Each node that can move to a community with strictly more neighbors does so in a single iteration. Implemented in Python, it can still exploit   a priori information about an initial partition, though it does not refine beyond one pass.

\item \textbf{Mirror}: also referred to as  zero-order hold (ZOH) in the realm of signal processing and  control theory, this approach returns the input partition unchanged (see~\ref{sec:onepass}). If $\noise$ is small, ``doing nothing'' can   outperform more complex moves. If $\noise$ is large, Mirror yields a poor partition.

\end{itemize}

\subsection{Reverse Engineering the Robustness of Ground Truth Partitions}
\label{sec:gt_robustness}

\begin{figure}[t]
    \includegraphics[width=1\linewidth]{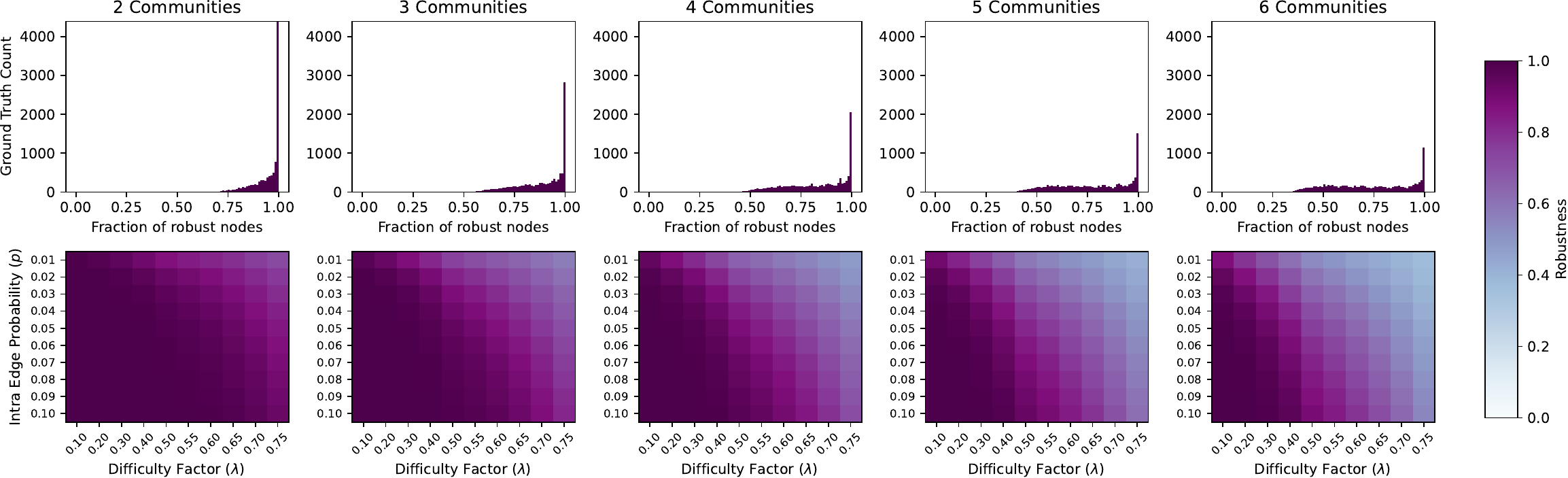}
    \caption{The top row displays histograms illustrating the distribution of the fraction of robust nodes across various parameter settings. The bottom row presents heatmaps, with each cell showing the average fraction of robust nodes for a given connection probability ($p$) and difficulty factor ($\difficulty$). This figure demonstrates how the robustness (defined in Eq.~\eqref{eq:robustness}) of the ground-truth partition changes with the difficulty factor. Each heatmap corresponds to a different number of communities ($K \in \{2,3,4,5,6\}$), keeping $n=1020$ fixed and $N=1020/K$. Darker colors indicate a higher fraction of fully-robust nodes.}
    \label{fig:gt_robust}
\end{figure}

Our goal in this section is to determine   whether the   robustness of a partition     correlates with its similarity to a ground truth reference partition. To this aim, we employ a reverse-engineering approach: we begin with known ground-truth partitions and calculate their robustness using the methodology in Section~\ref{sec:setup}.
The robustness of the ground truth is visualized using a heatmap (Figure~\ref{fig:gt_robust}), with $p$ and $\difficulty$ as its two dimensions. Each cell pair ($p$, $\difficulty$) represents networks generated with SAPPM using these parameters. The darkness of a cell  is proportional to its robustness, as given by Eq.~\eqref{eq:robustness}.
This analysis reveals several key insights, which are discussed in the sequel. 

\subsubsection{Robustness Decreases as  $K$ Grows}
\label{sec:robustness_decreases_k_grows}

As the number of communities ($K$) increases, the robustness of the ground truth partition tends to decrease. This is expected since the chance that the community of a given node is   the community where this node has   the highest number of neighbors and  fewest non-neighbors diminishes as $K$ increases. To illustrate that point, consider two extreme cases.  When $K=1$ we have one single feasible solution to the problem, the grand coalition, which in this case is always fully robust, as nodes have no alternative communities to transition to.  When $K=\nvertices$, we have multiple feasible solutions, including a set of singletons.  The latter, however,    is  never   fully robust for connected graphs, as  under the criterion of maximizing neighbors, nodes can   benefit by transitioning to a community with  additional nodes. 
  
\subsubsection{Effect of Parameters $p$ and $\difficulty$}
\label{sec:effect_params}

The robustness is higher for lower values of $\difficulty$ and higher values of $p$.  Indeed, we observe a gradual decline of robustness as $\lambda$ grows.  For $K=2$, the robustness remains high for a large range of $p$ and $\difficulty$. As $K$ increases, the robustness heatmap starts to show a gradual decline where higher values of $\difficulty$ and lower values of $p$ result in lower robustness.
This  occurs   because lower intra-community connection probabilities ($p$) and higher inter-community connection probabilities ($q = p \difficulty$) make communities less distinct and harder to detect. 
 
\subsubsection{Robust Partitions and Ground Truth}
\label{sec:robustness_of_gt}

Although not all the ground truth partitions  correspond to an equilibrium (see Definition~\ref{def:solution}), we observe that a significant fraction of those partitions does correspond to fully robust equilibria. These are the cells marked with dark blue in Figure~\ref{fig:gt_robust}, with robustness equal one.  In those cases, nodes have no incentive to transition across communities for any value of $\resolution$ (case 1(b) in Table~\ref{tab:node-stability}).  For the cases where a fraction of nodes has incentive to    deviate, such incentive holds for any $\resolution$ (case 1(a) in Table~\ref{tab:node-stability}, noting that $\resolution^{\star}=\Delta d \ge 1$ as the denominator in Eq.~\eqref{eq:resolution_lower_bound} is one when $n_{\clusterA}=n_{\clusterB}$).
 
\subsubsection{Summary}
\label{sec:summary_gt_robustness}

In this section we evaluated  the robustness of  ground-truth partitions. Our results provide   evidence that the ground truth partition in synthetic networks generated by SAPPM tends to be highly robust, motivating further studies on the use of the proposed  robustness criterion for partition selection purposes. 
In the following section, we further compare the considered hedonic game approach against state-of-the-art methods.



\subsection{Performance Evaluation in a Community Tracking Scenario}
\label{sec:community_tracking}

Having analyzed the inherent robustness of ideal, ground-truth partitions in Section~\ref{sec:gt_robustness}, we now shift our focus to evaluating the performance of the partitions produced by practical algorithms. To do this, we use a  community tracking scenario, a common task in the analysis of dynamic networks~\cite{greene2010tracking,avrachenkov2023recovering}.
In this setup, an existing, once-accurate partition has become outdated due to network evolution. The goal is to efficiently refine this ``stale" partition to match the network's current state. By systematically introducing noise into a known ground-truth partition, we can simulate this staleness and create a controlled benchmark to assess how different algorithms perform in terms of efficiency, the robustness of the partitions they find, and their accuracy in recovering the true community structure.%
\footnote{All experiments were conducted on an Intel\textsuperscript{\textregistered} Xeon\textsuperscript{\textregistered} E5-2670 processor. Code and implementation details at \url{https://github.com/lucaslopes/hedonic-game}.}


We use the network generator from Section~\ref{sec:synthetic_networks} to create synthetic networks with known ground-truth partitions. To simulate outdated information, we perturb the ground-truth partition, $\partition^{\star}$, by applying a noise parameter $\noise \in \{0.1, 0.25, 0.5, 0.75, 1.0\}$. This process randomly permutes a fraction $\noise$ of nodes across communities while keeping community sizes the same. When $\noise = 0$, the initial partition is identical to the ground truth. When $\noise = 1$, the partition is completely shuffled, meaning only the number of communities, $K$, is retained as prior information.

We analyze the impact of noise (Section~\ref{sec:impact_noise}) and the number of communities (Section~\ref{sec:impact_n_communities}) on these three metrics. We then examine the interplay between these dimensions (Section~\ref{sec:metrics_interplay}) and conclude with a summary of our findings (Section~\ref{sec:summary_tracking}).

\subsubsection{Impact of Noise on Efficiency, Robustness, and Accuracy}
\label{sec:impact_noise}

\begin{figure}[t]
    \centering
    \includegraphics[width=\linewidth]{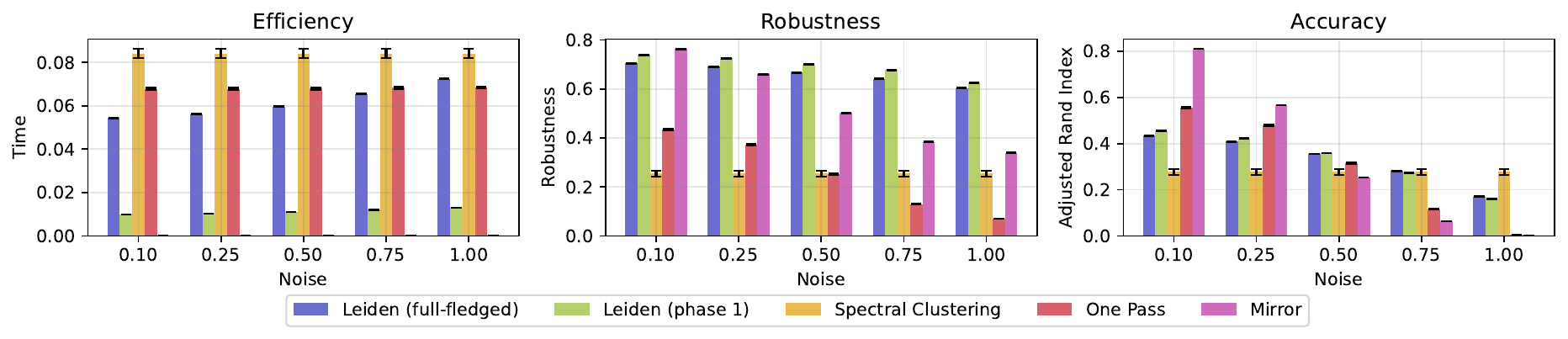}
    \caption{\textbf{Effect of noise on efficiency, robustness, and accuracy.} 
    Each barplot shows results for five noise levels ($x$-axis) and five methods (color-coded). \emph{Left:} Efficiency (runtime in seconds; lower is better). \emph{Center:} Robustness (fraction of nodes satisfying strict max-min criteria). \emph{Right:} Accuracy (Adjusted Rand index~\cite{hubert1985comparing} comparing final partition to the ground truth). The Mirror and  \emph{One Pass}  methods degrade significantly at higher noise, while both Leiden variants remain robust and accurate. Spectral clustering is insensitive to noise but achieves lower robustness.}
    \label{fig:noise}
\end{figure}

Figure~\ref{fig:noise} shows how each method behaves as a function of the noise, $\noise$. 

\paragraph{\textbf{Efficiency}}
The \emph{Mirror} baseline is trivially fastest, performing no updates. Next is \emph{Leiden (Phase~1)}, which converges in negligible time even as $\noise$ grows, in agreement with pseudo-polynomial convergence   (Theorem~\ref{THEO:EQUILIBRIUM}). 
In contrast, \emph{Leiden (Full-Fledged)} runs longer, especially for larger $\noise$, because its refinement and aggregation phases must correct more errors.  
 Additionally, \emph{Leiden (Phase~1)} is inherently distributed, making it well-suited for implementation in decentralized systems, whereas Phases 2 and 3 of \emph{Leiden (Full-Fledged)} may pose challenges for distributed implementation.
\emph{Spectral Clustering} is the slowest overall, as it ignores the initial partition and computes a leading eigenvector with $O(\nvertices^3)$ complexity.


As the name suggests, the \emph{One Pass}  method processes each of the $\nvertices$ nodes once. For this reason, its running time  corresponds to the first $\nvertices$ iterations of Leiden (Phase 1). 
%
%
Clearly,  \emph{One Pass}  is asymptotically faster than Leiden (Phase~1), and its slow convergence speed   in Figure~\ref{fig:noise}   serves to appreciate how the optimized implementation of Leiden (in C) outperforms the non-optimized implementation  of  \emph{One Pass}  (in Python). 

\paragraph{\textbf{Robustness}}
As $\noise$ grows, the \emph{Mirror} and \emph{One Pass} partitions become less robust: a random or partially incorrect initial partition has many nodes that could improve by switching communities. \emph{Spectral Clustering} yields consistently low robustness because it does not incorporate prior knowledge and does not account for a  robustness-related objective function. By contrast, \emph{Leiden (Phase~1)} and \emph{Leiden (Full-Fledged)} maintain   robustness above 0.6 even at $\noise=1$, as the CPM objective function   intrinsically captures  the goal of maximizing a  robustness-related metric. 


\paragraph{\textbf{Accuracy}}
When $\noise<0.5$, \emph{Mirror} and \emph{One Pass} often outperform more complex methods, simply by leaving a nearly correct partition intact or making a single pass of improvements. However, for $\noise \ge 0.5$  these naive methods degrade rapidly, being outperformed by  \emph{Leiden (Phase~1)} and \emph{Leiden (Full-Fledged)}. \emph{Spectral} also becomes more competitive at higher noise, but it lags behind Leiden in overall accuracy, indicating that leveraging partial ground-truth information is advantageous.

\subsubsection{Impact of the Number of Communities}
\label{sec:impact_n_communities}

\begin{figure}[t]
    \centering
    \includegraphics[width=\linewidth]{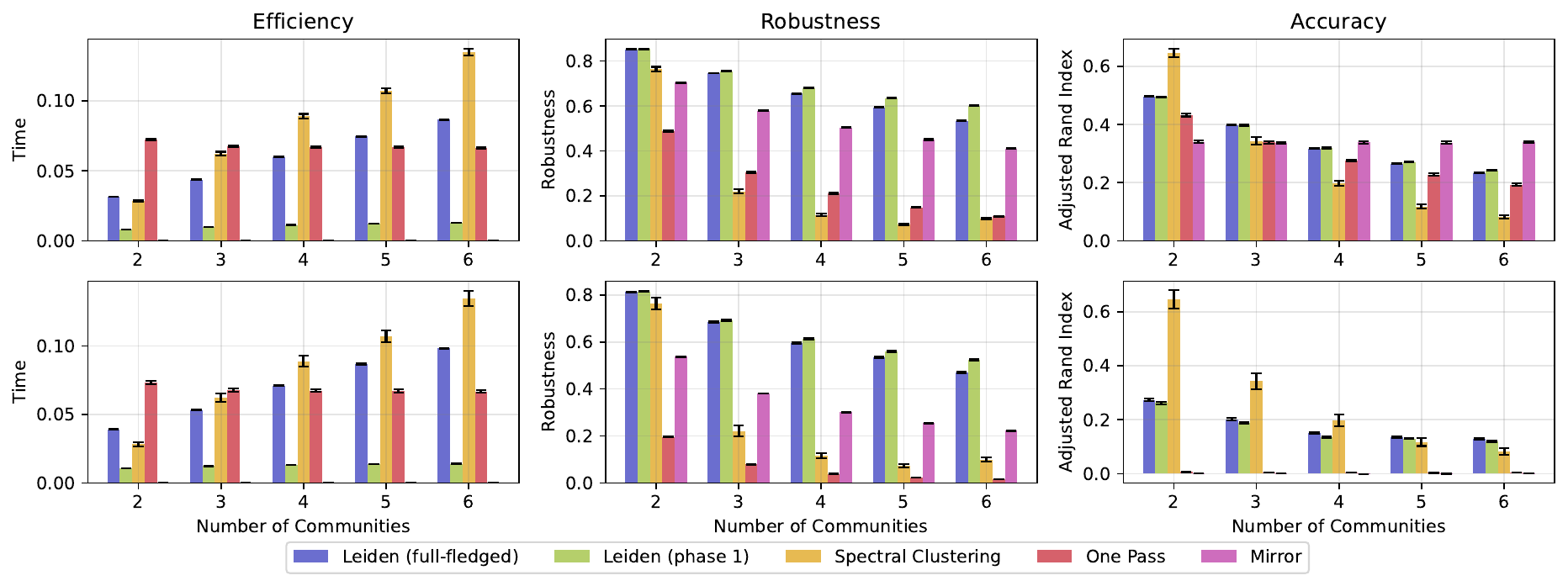}
    \caption{\textbf{Effect of the number of communities.} 
    Top row: average results over all noise levels. Bottom row: results restricted to $\noise=1$. \emph{Left column:} Efficiency. \emph{Center:} Robustness. \emph{Right:} Accuracy. Methods exploiting the initial partition (Leiden, One Pass, Mirror) benefit when $\noise<1$; under fully scrambled conditions ($\noise=1$), methods that can systematically reassign nodes (Leiden, Spectral) prevail for higher $K$.}
    \label{fig:n_communities}
\end{figure}

Figure~\ref{fig:n_communities} examines how performance scales with the number of communities $K$. The top row shows averages over all $\noise$, while the bottom row reports results for  $\noise=1$.

\paragraph{\textbf{Efficiency}}
The convergence time of 
\emph{Spectral Clustering} and \emph{Leiden (Full-Fledged)} both grow roughly linearly as $K$ increases, reflecting the more complex search space. By contrast, \emph{Leiden (Phase~1)} remains almost flat in runtime, converging quickly regardless of $K$. The \emph{One Pass} approach shows a slight decrease in runtime with larger $K$, presumably because it moves fewer nodes in a single iteration.

\paragraph{\textbf{Robustness}}
All methods lose some robustness as $K$ grows, since finer partitions can create more potential gains from rearranging nodes. \emph{Leiden (Phase~1)} retains the highest robustness, slightly above \emph{Leiden (Full-Fledged)}, whereas \emph{Spectral} and \emph{One Pass} degrade steeply. Notably, \emph{Mirror} shows intermediate robustness, representing the shuffled input. When $\noise=1$, \emph{One Pass} and \emph{Mirror}  can reach near-zero robustness at $K=6$.

\paragraph{\textbf{Accuracy}}
\emph{Spectral} is most accurate at $K=2$ (often exceeding $0.6$ for the ARI accuracy) but it accuracy drops for $K>2$. It is known that the accuracy of \emph{Spectral Clustering} decreases as the number of clusters increases \cite{lei2015consistency}.  The accuracy of \emph{Leiden (Phase~1)} and \emph{Leiden (Full-Fledged)} also  drops as $K$ grows, but such drop is more significant for higher values of  $K$. In particular,  \emph{Leiden (Phase~1)} and \emph{Leiden (Full-Fledged)} outperform  \emph{Spectral} for   $K\geq3$ when leveraging information about the initial partition. They also outperform \emph{Spectral} for  $K\geq5$ when $\noise=1$.  
The accuracy of 
\emph{Mirror} remains roughly stable with respect to $K$ when the noise is moderate or low (impacting the results reported in the  top row of Figure~\ref{fig:n_communities}). Under fully scrambled conditions ($\noise=1$), \emph{Mirror} and \emph{One Pass} end up with similarly poor accuracy.

\subsubsection{Accuracy Versus Efficiency and Robustness}
\label{sec:metrics_interplay}

\begin{figure}[t]
    \centering
    \includegraphics[width=\linewidth]{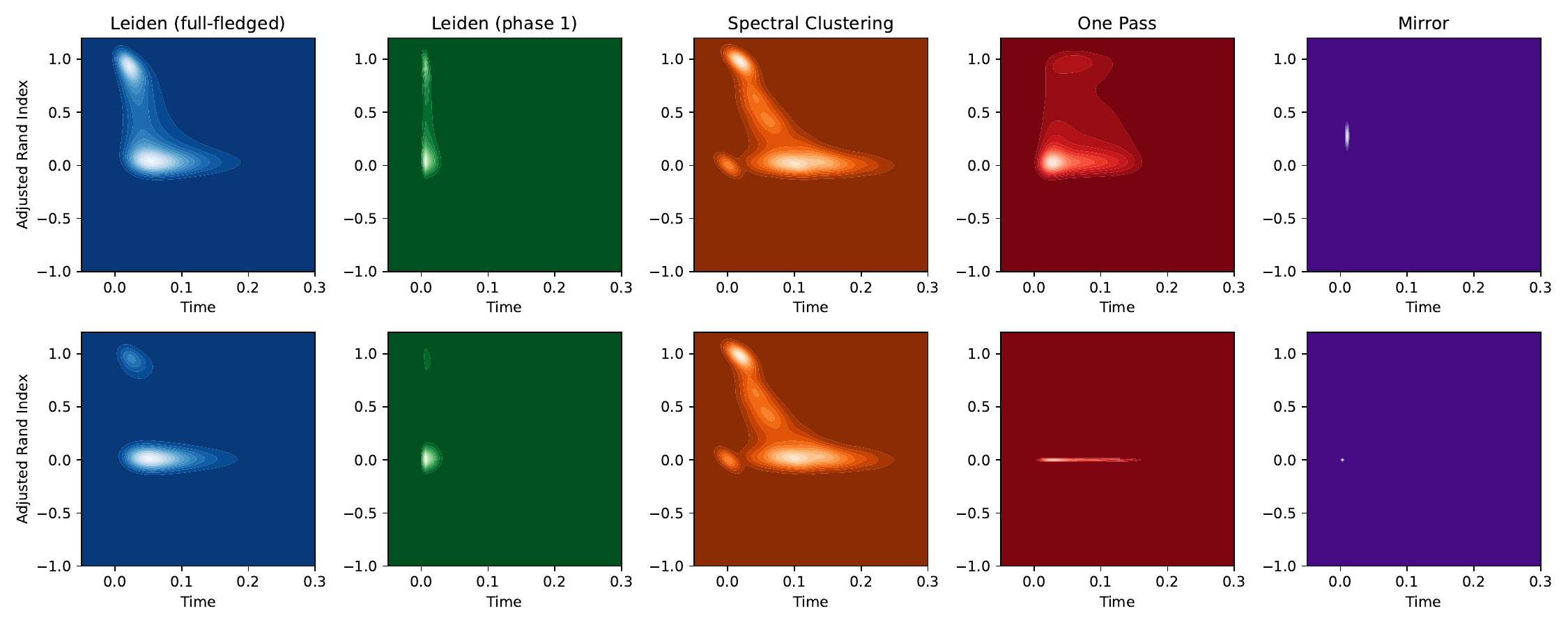}
    \caption{\textbf{Accuracy vs.\ efficiency.} Each panel is a density contour for one method, with the top row averaging all $\noise$ and the bottom row restricted to $\noise=1$. The $x$-axis is runtime in seconds, and the $y$-axis is the Adjusted Rand index (ARI)~\cite{hubert1985comparing}. A method is ideal if it is \emph{both} highly accurate (top) and efficient (left).}
    \label{fig:acc_efficiency}
\end{figure}

\begin{figure}[t]
    \centering
    \includegraphics[width=\linewidth]{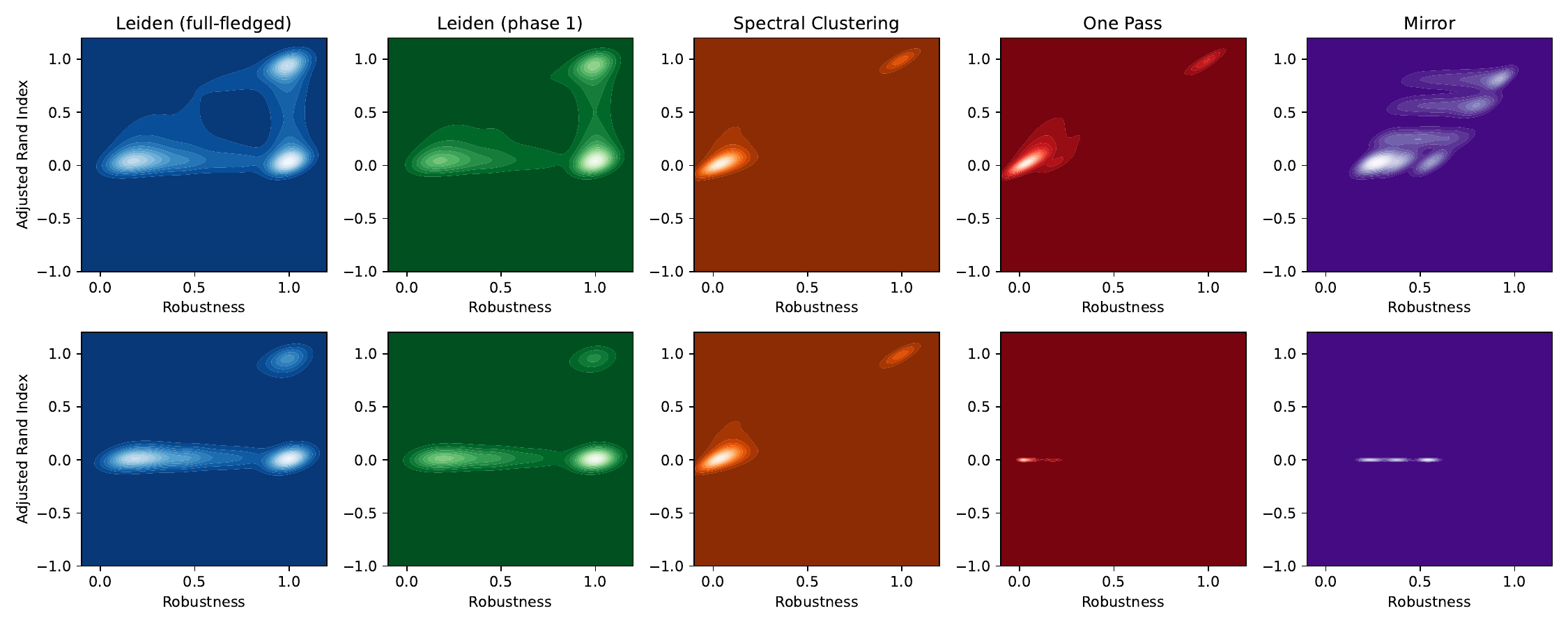}
    \caption{\textbf{Accuracy vs.\ robustness.} Each panel is a density contour for one method, with the top row averaging all $\noise$ and the bottom row restricted to $\noise=1$. The $x$-axis is the fraction of robust nodes (Eq.~\eqref{eq:robustness}), and the $y$-axis is the Adjusted Rand index (ARI)~\cite{hubert1985comparing}. A method is ideal if it is both highly robust (right) and accurate (top).}
    \label{fig:acc_robustness}
\end{figure}

To visualize the \emph{joint} interplay between metrics, Figures~\ref{fig:acc_efficiency} and~\ref{fig:acc_robustness} show 2D contour plots. The top row in each figure aggregates all noise levels, and the bottom row focuses on $\noise=1$. Each column corresponds to a specific method.

\paragraph{\textbf{Leiden (Phase~1) vs.\ Leiden (Full-Fledged)}}
Both present a ``two-cluster'' pattern in the  accuracy-efficiency space (Figure~\ref{fig:acc_efficiency}): one cluster has high accuracy  and low convergence time,  and the other correspond to   less accurate solutions and higher convergence times. \emph{Leiden (Phase~1)} is notably more concentrated at low runtimes, reflecting its single-phase nature. In the accuracy-robustness space (Figure~\ref{fig:acc_robustness}), they both form ``double humps'' with high densities at robust, accurate partitions, plus some spread at moderate accuracy.

\paragraph{\textbf{Spectral Clustering}}
Because it ignores initial conditions, its performance does not shift with $\noise$ (identical results in top and bottom rows in Figures~\ref{fig:acc_efficiency} and~\ref{fig:acc_robustness}). We see a band of points around low  accuracy and widely varying runtimes (depending on $K$). 
In the accuracy-robustness view, \emph{Spectral} yields a cluster of points  with low robustness  and accuracy (ARI $\leq0.2$) and another cluster where both  robustness and accuracy are  higher (ARI $\geq0.6$). The latter accounts, for instance, to scenarios where $K=2$, as indicated in Figure~\ref{fig:n_communities}. 

\paragraph{\textbf{One Pass}}
The \emph{One Pass}  approach tends to produce partitions of middling accuracy (ARI $\approx 0$). Under $\noise=1$, it can be slower if the network is difficult to partition (e.g., large $K$ with weak community structure). In the accuracy-robustness space, \emph{One Pass} can form small ``peaks'' of higher accuracy and robustness when $\noise<0.5$. Otherwise, ARI is roughly zero.

\paragraph{\textbf{Mirror}}
Its runtime is effectively zero, so Figure~\ref{fig:acc_efficiency} shows vertical stripes when the runtime is close to zero. Accuracy then depends solely on $\noise$. In Figure~\ref{fig:acc_robustness}, \emph{Mirror}  yields a broad distribution in the robustness-accuracy space: if the input partition is close to ground truth, we see top-right peaks; if heavily shuffled (bottom row), accuracy is near zero  and robustness varies between   0.1 and 0.6.

\subsubsection{Summary of Numerical Findings}
\label{sec:summary_tracking}

The reported  experiments illustrate the practical benefits of hedonic-game-based methods for \emph{community tracking}, where an outdated but partially correct partition is available. While simpler baselines, such as  \emph{Mirror} and \emph{OnePass}, can suffice for low noise, they produce partitions with low accuracy  at higher noise. \emph{Spectral} is a good candidate for $K=2$ but is not competitive for the task of community tracking  otherwise, as it ignores  valuable prior information. By contrast, \emph{Leiden (Phase~1)} and \emph{Leiden (Full-Fledged)} consistently yield robust and accurate solutions in short runtime, confirming the theoretical insights of Section~\ref{sec:analytical_contributions}.  In dynamic or real-world settings, starting from an imperfect partition and running a fast local improvement can be highly effective, especially if noise is moderate so that prior knowledge about  the partition is only partially outdated.

\section{Conclusion}
\label{sec:conclusion}

Community detection is a fundamental component in modern data science. In this work, we contributed to the community detection literature by establishing that the Constant Potts Model (CPM) aligns with a hedonic potential function in coalitional games. Using the Leiden algorithm, we identified CPM local maximizers that correspond to equilibria of the hedonic game. Additionally, we provided theoretical conditions under which the Leiden algorithm converges in pseudo-polynomial time.

We analyzed two stability criteria: one enforcing strict maximization of neighbors and minimization of non-neighbors within communities, and another applying a relaxed utility function based on a weighted sum of the  number of neighbors and non-neighbors, with the resolution parameter $\resolution$ acting as the weight. We introduced the concept of robust nodes, i.e., nodes that have no incentive to switch communities for any value of the resolution parameter.

Our experiments demonstrated that the fraction of robust nodes in an equilibrium can serve as an indicator for extracting high-quality solutions from the output of the Leiden algorithm, which maximizes the relaxed $\resolution$-weighted utility function. This insight enables a principled selection among multiple equilibria, improving the overall accuracy of community detection.

Looking ahead, our research opens multiple directions for further exploration. While the algorithms studied in this work primarily identify non-overlapping communities, extending them to detect overlapping communities remains an important avenue for future research~\cite{chen2010game}. This could be achieved by leveraging multiple initializations of the hedonic game. Additionally, we plan to explore how different strategies for scheduling node movements between communities impact final results, particularly considering the effects of non-greedy transitions on detection quality.

Furthermore, we anticipate leveraging tools from evolutionary game theory to enhance solution quality~\cite{durand2018distributed,kleer2019topological,amiet2019pure}. Strengthening the connection between community detection and game-theoretic approaches presents a promising path for refining equilibrium selection and improving the robustness of community structures in dynamic environments.

\begin{appendix}

\section{Table of Notation}
\label{sec:notation}

Table~\ref{tab:notation} provides a comprehensive list of the symbols and terms used throughout this work, along with their definitions.

\begin{table}[h!]
\centering
\caption{Table of Notation}
\begin{tabular}{|c|l|}
\hline
\textbf{Symbol} & \textbf{Definition} \\ \hline
$\graph = (\vertices, \edges)$ & Network (graph) with $\nvertices=|\vertices|$ vertices and $\medges=|\edges|$ edges. \\ \hline
$\partition$ & Partition of a network into disjoint communities. \\ \hline
$K$ & Number of communities. \\ \hline
$\community_k$ & The $k$-th community of a partition. \\ \hline
$\clusterA$ & An arbitrary community to where the node considers moving to. \\ \hline
$\clusterB$ & An arbitrary community from where the node considers leaving from. \\ \hline
$n_k$ & Number of nodes within community $\community_k$. \\ \hline
$m_k$ & Number of edges within community $\community_k$. \\ \hline
$\nodecomm_\anode$ & Index of the community assigned to node $\anode$. \\ \hline
$d^{k}_\anode$ & Number of neighbors of node $\anode$ in  community $k$. \\ \hline
$\hat{d}^{k}_\anode$ & Number of non-neighbors of node $\anode$ in   community $k$. \\ \hline
$v_{\anode\bnode}$ & The \textit{hedonic} value for a pair of nodes $\anode$ and $\bnode$. \\ \hline
$R(i,\partition)$ & Robustness of node $\anode$ for a giving partition $\partition$. \\ \hline
$\resolution$ & Resolution parameter that controls communities granularity. \\ \hline
$\noise$ & Noise parameter simulating outdated ground-truth knowledge in community tracking tasks. \\ \hline
$\delta(c)$ & Indicator function, returning 1 if the condition $c$ holds, and 0 otherwise. \\ \hline 
$\potentialNode_{\anode}^{\resolution}(\community_k)$ & Node potential: contribution of node $\anode$ to its community  (superscript dropped  \\
& $\quad$ when clear from context). \\ \hline
$\potentialCoalition^{\resolution}(\community_k)$ & Community potential: sum of the contributions of all nodes in a community. \\ \hline
$\potentialPartition^{\resolution}(\pi)$ & Partition potential: sum of the community potentials across the entire partition. \\ \hline
\end{tabular}
\label{tab:notation}
\end{table}

\section{The Metagraph of Partitions}
\label{sec:metagraph}

To visualize the entire decision landscape available to agents, we construct a \textit{metagraph} that maps the universe of possible network partitions into a comprehensive graph structure (Figure~\ref{fig:metagraph}). In this metagraph, each vertex represents a unique partition of the underlying network $\graph$, and edges represent rational, single-node transitions between them. The construction follows a systematic, three-step process: generating the metanodes, identifying the edges, and typing the edges based on agent incentives.

\subsection{Metanodes: The Universe of Partitions and Layout}
The metanodes of the metagraph represent all possible unique partitions of the $\nvertices$ nodes in the original graph $\graph$.
The metanodes of the metagraph represent all $B_\nvertices$ possible unique partitions of the $\nvertices$ nodes in the original graph, where $B_\nvertices$ is the corresponding Bell number. We generate this complete set by systematically enumerating every possible grouping of the nodes.

For analytical clarity and structured visualization, two specific partitions are designated as anchors for the layout algorithm:
\begin{itemize}
    \item \textbf{The Grand Coalition}: The partition consisting of a single community that includes all $\nvertices$ nodes. This is represented by the blue node in Figure~\ref{fig:metagraph}.
    \item \textbf{The Singleton Partition}: The partition where each of the $\nvertices$ nodes forms its own individual community. This corresponds to the red node in the figure.
\end{itemize}
The positions of all other partitions (intermediate vertices) are determined by a custom layout algorithm that organizes them into layers based on their \textbf{move distance} from these two anchors. The move distance between two partitions is the minimum number of single-node moves required to transform one into the other. This organizes the metagraph logically, often revealing an evolutionary path from full integration (the grand coalition) to complete fragmentation (the singletons).

\subsection{Edges: Single-Node Transitions}
An edge exists between two metanodes in the metagraph (representing partitions $\partition_1$ and $\partition_2$) if and only if one can be transformed into the other by moving a single node from one community to another. This defines a neighborhood relationship on the space of partitions.
These connections are identified by systematically comparing all pairs of partitions and creating an edge if they differ by the relocation of exactly one node.

\subsection{Edge Types: Non-Frustrated vs. Frustrated Choices}
The character of each edge in the metagraph is determined by the nature of the choice available to the moving node. A move is evaluated based on the trade-off between gaining intra-community neighbors (``friends'') and avoiding intra-community non-neighbors (``strangers''). This trade-off is precisely quantified by the Familiarity Index, which serves as a decision threshold:
\begin{equation}
    \resolution^\star = \frac{\Delta d}{\Delta d + \Delta \hat{d}} \label{eq:familiarity_index_appendix}
\end{equation}
where $\Delta d$ represents the change in the number of friends and $\Delta \hat{d}$ is the change in the number of strangers for the moving node. The nature of the choice is defined as follows:

\begin{itemize}
    \item \textbf{Unidirectional Edge (Non-Frustrated Choice)}: A directed edge is drawn if the choice is unambiguous, which occurs when $\resolution^\star \notin (0, 1)$. This corresponds to a clear-cut move where a node can improve one metric (e.g., gain friends) without worsening the other (e.g., by not gaining any strangers). With no conflict, the preference is clear, resulting in a single-colored (purple) directed edge.

    \item \textbf{Bidirectional Edge (Frustrated Choice)}: A bidirectional edge is drawn if the choice is \textit{frustrated}, which occurs when $\resolution^\star \in (0, 1)$. This signifies a direct trade-off: the node must choose between a community with more friends but also more strangers, or one with fewer of both. This conflict is visualized with a two-colored edge, where the blue path points toward the partition offering more friends and the red path points toward the one offering fewer strangers. The value of $\resolution^\star$ represents the critical value of the resolution parameter $\resolution$ at which an agent is indifferent.
\end{itemize}

By constructing the metagraph in this manner, we create a complete map of the local decision landscape. It visualizes every possible single-step move, clearly distinguishing between straightforward improvements (non-frustrated choices) and the complex, trade-off-dependent decisions (frustrated choices) that are central to the community detection problem.

\section{Community Detection Algorithms}
\label{sec:methodsappendix}

For the sake of completeness, in this appendix we report various community detection algorithms. 
Some of the  algorithms, such as Spectral Clustering and the Louvain Method, do not accept an initial partitioning as an argument. In contrast, other methods, including the Leiden Algorithm,  
can leverage an initial configuration for community detection.
%
%
Regarding practical implementation aspects, the Spectral, Louvain, and Leiden algorithms have highly optimized implementations  in C, with the Python \texttt{igraph} interface serving as a wrapper.
%
%

One of the key parameters considered in our work is the resolution parameter $\resolution$.   \cite{bogomolnaia2002stability, traag2011cpm, traag2019leiden, avrachenkov2017cooperative} highlight that several community detection algorithms optimize an objective function with a resolution parameter.  Louvain and Leiden  are two examples of such algorithms, as further detailed in the sequel. 

\subsection{Spectral Clustering}
\label{sec:spectral}

Spectral clustering methods are known to be asymptotically optimal for a broad class of networks~\cite{von2007tutorial,abbe2015exact, chan1997optimality}, 
and the origins of spectral clustering are well-documented in the literature~\cite{cheeger1970lower,  shi2000normalized, ng2001spectral}. The technique gained widespread recognition in the machine learning community through the works of Shi and Malik~\cite{shi2000normalized} and Ng, Jordan, and Weiss~\cite{ng2001spectral}.

While spectral methods produce high-quality communities, their applicability to large-scale problems is limited by their computational complexity of $O(\nvertices^3)$. These methods allow for a predefined number of clusters but lack the flexibility to incorporate an arbitrary initial partition. Furthermore, spectral clustering does not include a resolution parameter, which limits its adaptability to different levels of granularity.%
\footnote{In this work, we use the \texttt{community\_leading\_eigenvector()} method from the \texttt{igraph} library to perform spectral clustering: \url{https://python.igraph.org/en/0.10.8/api/igraph.Graph.html\#community_leading_eigenvector}.}


\subsection{Louvain Algorithm}
\label{sec:method_louvain}

The Louvain algorithm belongs to the class of algorithms that maximize a modularity function. The seminal innovation of the Louvain method~\cite{blondel2024louvain}, named after the university where it was developed, lies in its \emph{aggregation phase}, which compresses the network by merging each detected community into a single supernode. In this phase, the edge weight between any two supernodes is recalculated as the sum of the weights of all edges connecting the corresponding communities in the original graph. This multilevel framework enhances scalability by iteratively reducing the graph’s complexity while also revealing hierarchical community structures at different scales. However, this approach may sometimes obscure internal connectivity flaws by merging nodes into communities that are not entirely cohesive.

The Louvain method  extracts non-overlapping communities from large networks and operates in two main phases: (1) moving nodes between communities to optimize modularity and (2) aggregating the graph by merging detected communities into single nodes. After each iteration, the newly formed communities become nodes, and the process repeats until convergence. Due to this layered approach, the method is also referred to as a multi-level method.

The time complexity of the Louvain algorithm is an open subject. While it is   cited as ``essentially linear in the number of links of the graph''~\cite{lancichinetti2009community}, to the best of our knowledge this complexity was empirically observed rather than theoretically proven.

  As a modularity-based approach, the Louvain method is widely used in practice due to its efficiency and empirical performance. 
Similar to the Spectral method, the Louvain method does not accept an initial partition as an input parameter. However, unlike Spectral clustering, the Louvain method includes a resolution parameter, allowing control over the granularity of the detected communities.

\subsection{Leiden Algorithm}
\label{sec:leiden}

The Leiden algorithm~\cite{traag2019leiden}, named after its originating university, improves upon the Louvain algorithm~\cite{blondel2024louvain} by addressing its most significant shortcoming: the tendency to produce arbitrarily badly connected and even disconnected communities (see Section~\ref{sec:resolution_limit}). By introducing a sophisticated refinement phase, Leiden not only guarantees that all communities are well-connected but also identifies higher-quality partitions more efficiently.

The full algorithm iterates through a three-phase process: (1) local moving of nodes, (2) refinement of the partition, and (3) aggregation of the network. This structure allows Leiden to explore the solution space more effectively while providing explicit guarantees about the quality of the resulting communities.

\paragraph{\textbf{Phase 1: Local Moving of Nodes}}
This phase is conceptually similar to Louvain's local optimization but employs a more efficient "fast local move" procedure. Instead of repeatedly iterating over all nodes, the Leiden algorithm maintains a queue of nodes to visit. Initially, all nodes are placed in the queue in a random order. When a node is moved to a new community, only its neighbors (who are not in the node's new community) are added to the queue if they are not already present. This targeted approach avoids redundant checks on nodes whose neighborhoods have not changed, making the local moving phase significantly faster than Louvain.

\paragraph{\textbf{Phase 2: Refinement of the Partition}}
This is the central innovation of the Leiden algorithm, designed to ensure that communities are internally well-connected. After the local moving phase yields a partition $\partition$, the refinement phase aims to find a refined partition, $\partition_{\text{refined}}$, where communities in $\partition$ may be split into multiple, more cohesive subcommunities. The refinement operates on each community from $\partition$ individually, merging nodes locally only within that community's boundaries. To better explore the partition space, the algorithm does not necessarily make the greediest move; instead, it may randomly select any merge that results in a non-negative quality improvement.

\paragraph{\textbf{Phase 3: Aggregation of the Network}}
In the final phase, the network is aggregated for the next level of optimization. Leiden uses the \textit{refined partition} $\partition_{\text{refined}}$ to construct the nodes of the new aggregate network. However, the initial community assignments for this new network are based on the original, \textit{non-refined partition} $\partition$. This design allows the algorithm to break apart poorly formed communities while preserving the larger-scale structure for the next iteration.

\paragraph{\textbf{Implementation and Scope in This Work}}
In our experiments, we use the \texttt{community\_leiden()} method from the \texttt{igraph} library, optimizing the Constant Potts Model (CPM) quality function. We found that setting the resolution parameter $\resolution$ to the graph's edge density provided accurate results.

While the full Leiden algorithm includes all three phases, in this work, we focus specifically on the local-move dynamics of Phase 1 to analyze its convergence properties as a hedonic game. We therefore treat Phases 2 and 3 as components of the complete algorithm but omit them from our direct analysis.

\subsection{One-Pass Improvement}
\label{sec:onepass}

The one-pass improvement strategy is a simple yet highly efficient approach that serves as a baseline for comparison. The algorithm functions as follows: given an initial partition, it identifies all nodes that would increase their number of neighbors within a community by moving to the community where they have the highest number of such neighbors. Once the nodes to be relocated and their target communities are determined, the algorithm moves all marked nodes in a single step.
This method, referred to as one-pass improvement~\cite{abbe2015exact}, has a time complexity of $O(n)$. Like state-of-the-art methods such as Leiden and hedonic games, it accepts an initial partition, making it suitable for community tracking. However, similar to the classical spectral method, it lacks a resolution parameter, limiting its flexibility in tuning community granularity.

\end{appendix}

\textbf{Acknowledgments. }   This work was partially supported by CNPq, CAPES and FAPERJ through grants JCNE E-26/201376/2021 and CNE E-26/204.268/2024.

\bibliographystyle{elsarticle-num} 
\bibliography{main.bib}

\end{document}